\documentclass[letterpaper]{article}% DO NOT CHANGE THIS
\usepackage{aaai23}
\usepackage{times}  % DO NOT CHANGE THIS
\usepackage{helvet}  % DO NOT CHANGE THIS
\usepackage{courier}  % DO NOT CHANGE THIS
\usepackage[hyphens]{url}  % DO NOT CHANGE THIS
\usepackage{graphicx} % DO NOT CHANGE THIS
\usepackage{natbib}  % DO NOT CHANGE THIS AND DO NOT ADD ANY OPTIONS TO IT
\usepackage{caption} % DO NOT CHANGE THIS AND DO NOT ADD ANY OPTIONS TO IT
\usepackage{algorithm}
\usepackage{algorithmic}

\usepackage{newfloat}
\usepackage{listings}
\DeclareCaptionStyle{ruled}{labelfont=normalfont,labelsep=colon,strut=off} % DO NOT CHANGE THIS
\floatstyle{ruled}
\newfloat{listing}{tb}{lst}{}
\floatname{listing}{Listing}
\usepackage{amsmath}
\usepackage{amssymb}
\usepackage{mathtools}
\usepackage{amsthm}
\usepackage{comment}
\usepackage{booktabs}

\usepackage{enumitem}
\usepackage{comment}
\usepackage{amsfonts,bm}
\usepackage{amsthm}
\usepackage{threeparttable}
\usepackage{caption}
\usepackage{float}
\usepackage{color}
\usepackage{subcaption}
\usepackage{multicol}

\theoremstyle{plain}
\newtheorem{theorem}{Theorem}
\newtheorem{proposition}{Proposition}
\newtheorem{lemma}{Lemma}
\newtheorem{corollary}{Corollary}
\theoremstyle{definition}
\newtheorem{definition}{Definition}
\newtheorem{assumption}{Assumption}
\theoremstyle{remark}

\newtheorem{step-subsection}{Step}
\newtheorem{step}{Step}
\newcommand{\argmax}{\mathop{\rm argmax}\limits}
\newcommand{\argmin}{\mathop{\rm argmin}\limits}

\newcommand{\var}{\text{Var}}
\newcommand{\Var}{\mathrm{Var}}

\usepackage{chngcntr}
\usepackage{apptools}
\AtAppendix{\counterwithin{lemma}{section}}
\AtAppendix{\counterwithin{assumption}{section}}
\AtAppendix{\counterwithin{step-subsection}{subsection}}
\AtAppendix{\counterwithin{step}{section}}
\AtAppendix{\counterwithin{proposition}{section}}
\AtAppendix{\counterwithin{corollary}{section}}
\AtAppendix{\counterwithin{theorem}{section}}
\AtAppendix{\counterwithin{example}{section}}
\AtAppendix{\counterwithin{table}{subsection}}

\title{Counterfactual Learning with General Data-generating Policies}
\author {
    Yusuke Narita, \textsuperscript{\rm 1}
    Kyohei Okumura, \textsuperscript{\rm 2}
    Akihiro Shimizu, \textsuperscript{\rm 3}
    Kohei Yata \textsuperscript{\rm 4}
}
\affiliations {
    \textsuperscript{\rm 1} Yale University,\\
    \textsuperscript{\rm 2} Northwestern University,\\
    \textsuperscript{\rm 3} Mercari, Inc.,\\
    \textsuperscript{\rm 4} University of Wisconsin-Madison,\\
    yusuke.narita@yale.edu,
    kyohei.okumura@u.northwestern.edu,
    akihiro-shimizu@mercari.com,
    yata@wisc.edu
}

\begin{document}

\maketitle

%\begin{comment}
\begin{abstract}
Off-policy evaluation (OPE) attempts to predict the performance of counterfactual policies using log data from a different policy. We extend its applicability by developing an OPE method for a class of both full support and deficient support logging policies in contextual-bandit settings. This class includes deterministic bandit (such as Upper Confidence Bound) as well as deterministic decision-making based on supervised and unsupervised learning. We prove that our method's prediction converges in probability to the true performance of a counterfactual policy as the sample size increases. We validate our method with experiments on partly and entirely deterministic logging policies. Finally, we apply it to evaluate coupon targeting policies by a major online platform and show how to improve the existing policy.
\end{abstract}
%\end{comment}

\section{Introduction}\label{intro}

In bandit and reinforcement learning, off-policy (batch) evaluation attempts to estimate the performance of some counterfactual policy given data from a different logging policy.
% \footnote{Key prior studies include \citep{li2010contextual,Strehl2010,li2011unbiased,Li2012,bottou2013counterfactual,Dudik2014,Swaminathan2015,Swaminathan2015b,Schnabel2016recommend,wang2016optimal,swaminathan2017off,narita2018efficient,bennett2019confounder, uehara2020covariate,su2020shrinkage,saito2021open} for bandit, and \citep{precup2000eligibility, Jiang16, Thomas16, liu2018representation, Farajtabar2018MoreRD, Irpan2019OffPolicyEV, kallus2019DRL, Uehara2020MWL} for reinforcement learning.
% }
Off-policy evaluation (OPE) is essential when deploying a new policy might be costly or risky, such as in education, medicine, consumer marketing, and robotics. OPE relates to other fields that study counterfactual/causal reasoning, such as statistics and economics.

Most existing OPE studies focus on \emph{full support} logging policies, which take all actions with positive probability in any context, such as stochastic bandit (e.g. $\epsilon$-greedy and Thompson Sampling) and random A/B testing.
However, real-world decision-making often uses \emph{deficient support} logging policies, including deterministic bandit (e.g. Upper Confidence Bound) as well as deterministic decision-making based on predictions obtained from supervised and unsupervised learning.
An example in the latter group is a policy that greedily chooses the action with the largest predicted reward.
OPE is difficult with a deficient support logging policy, since its log data contain no information about the reward from actions never chosen by the logging policy.
There appears to be no established OPE estimator for deficient support logging policies 
% \cite{sachdeva2020deficient}
\cite{Sachdeva2020-qi}.

We provide a solution to this problem. Our proposed OPE estimator is applicable not only to full support logging policies but also to deficient support ones.
%for stochastic and deterministic logging policies.
%We show that OPE is possible with log data from a wide class of machine learning decision algorithms (logging policies), and provide a general method for conducting OPE.
%This class includes stochastic policies such as contextual $\epsilon$-greedy and Thompson Sampling, as well as their non-contextual analogs and random A/B testing.
%Importantly, this class also includes deterministic policies such as Upper Confidence Bound (UCB) as well as deterministic decision rules based on predictions obtained from supervised and unsupervised learning.
We also allow for hybrid stochastic and deterministic logging policies, i.e., logging policies that choose actions stochastically for some individuals and deterministically for other individuals.
%We develop an OPE estimator that is usable for the general class of logging policies.
%To do so, we use

{\bf Method.}
Our OPE estimator is based on a modification of the Propensity Score \citep{rosenbaum1983central}, which we dub the ``Approximate Propensity Score'' (APS) \citep{narita2021algorithm}.
APS of action (arm) $a$ at context (covariate) value $x$ is the average probability that the logging policy chooses action $a$ over a shrinking neighborhood around $x$ in the context space.
If two actions have nonzero APS at $x$, the logging policy chooses both actions locally around $x$.
This enables us to estimate the difference in the mean reward between the two actions by exploiting the local subsample around $x$.
When the logging policy is deterministic, the subsample consists of individuals near the decision boundary between the two actions.
%We show that if APS at $x$ is nonzero for a pair of actions, it is possible to estimate the difference in the mean reward between the two actions by exploiting the local subsample around $x$.
We then use the estimated reward differences to construct an estimator for the performance of any given counterfactual policy. %The general class is thus defined by the existence of APS.
%\textcolor{red}{***It may be better to remove or rewrite the last three sentences. The current writing seems difficult for the reader to understand.***}

As the main theoretical result, we prove that our proposed OPE estimator is consistent. That is, the estimator converges in probability to the true performance of a counterfactual policy as the sample size increases, under the assumption that the mean reward differences are constant over the context space (Theorem \ref{cor:opeconsistency}).
This result holds whether the logging policy is of full support or deficient support. The proof exploits results from differential geometry and geometric measure theory, which have not been applied in machine learning research as far as we know.

{\bf Simulation Experiments.}
We validate our method with two simulation experiments.
The first considers a mix of full support and deficient support policies as the logging policy. Actions are randomly chosen for a small A/B test segment of the population and are chosen by a deterministic supervised learning algorithm for the rest of the population.
For the task of evaluating counterfactual policies, our method produces smaller mean squared errors than a baseline estimator that only uses the A/B test subsample.
The second experiment considers a situation in which we have a batch of data generated by a deterministic bandit algorithm.
We find that our estimator outperforms a regression-based estimator in terms of mean squared errors.
%for the task of evaluating counterfactual policies.

{\bf Real-World Application.}
We empirically apply our method to evaluate and optimize coupon targeting policies. Our application is based on proprietary data provided by Mercari Inc., a major e-commerce company running online C2C marketplaces in Japan and the US.
This company uses a deterministic policy based on uplift modeling to decide whether they offer a promotional coupon to each target customer.
We use the data produced by their policy and our method to evaluate a counterfactual policy that offers the coupon to more customers.
Our method predicts that the counterfactual policy would increase revenue more than the cost of coupon offers, suggesting that redesigning the current policy is profitable.

%\subsection{Related Work}
{\bf Related Work.}
Widely-used OPE methods include inverse probability weighting (IPW) \citep{precup2000eligibility,Strehl2010}, self-normalized IPW \citep{Swaminathan2015b}, Doubly Robust \citep{Dudik2014}, and more advanced variants \citep{wager2017estimation,Farajtabar2018MoreRD,su2020shrinkage}.
These methods are based on importance sampling (IS) and require that the logging policy be of full support, i.e., assign a positive probability to every action potentially chosen by the counterfactual policy. This restriction makes them hard to use when the logging policy is of deficient support.

There are two existing approaches to deficient support logging policies.\footnote{Sachdeva et al.~(2020) also proposes another approach in which they restrict the policy space.}
The first approach considers a logging policy that varies over time or across individuals \citep{Strehl2010}.
Viewing the sequence of varying logging policies as a single full support logging policy, it is possible to apply IS-based OPE methods.
Unlike this approach, our approach is usable even when the logging policy is fixed.

The second approach, called the Direct Method or Regression Estimator, predicts the mean reward conditional on the action and context by supervised learning and uses the prediction to estimate the performance of a counterfactual policy \citep{beygelzimer2009offset,Dudik2014}.
Similar regression-based methods are proposed for reinforcement learning settings \citep{duan2020minimax}.
This approach is sensitive to the accuracy of the mean reward prediction. It may have a large bias if the regression model is not correctly specified.
This issue is particularly severe when the logging policy is of deficient support, since each action is observed only in a limited area of the context space.
Our approach instead predicts the mean reward differences between actions by exploiting local subsamples near the decision boundaries without specifying the regression model.
\citet{narita2021algorithm} originally develop and empirically apply this approach in the context of treatment effect estimation with a binary treatment. 
This paper extends their approach to OPE with multiple actions. 
This idea relates to regression discontinuity designs in the social sciences \citep{lee2010rdd}.

It is worth noting that our approach is applicable to \emph{off-policy selection}, in which the researcher is to design a decision rule to select a policy given a finite set of policies \citep{Kuzborskij2021-vc}.
Since our method can estimate the expected reward of the policies, we can first estimate the reward of each, and then choose the one with the highest expected reward.

\section{Framework}
%{\bf Data Generating Process.}
${\cal A} \coloneqq \{1,...,m\}$ is a set of \textit{actions} that the decision maker can choose from.
Let $\mathbb R^p$-valued random variable $X$ denote the \textit{context}
%(e.g., the user's demographic profile and browsing history)
that the decision maker observes when picking an action.
% , where $p$ is the number of context variables.
%We think of $(Y(\cdot),X)$ as a random vector with unknown distribution.
Let ${\cal X}$ denote the support of $X$.
To simplify the exposition, we assume that $X$ is continuously distributed. %absolutely continuous with respect to the Lebesgue measure.
%We denote the chosen action by $A\in{\cal A}$.
% Let $Y(\cdot): {\cal A}\rightarrow \mathbb{R}$ denote a potential reward function that maps actions into rewards, where $Y(a)$ is the reward when action $a$ is chosen.
%(e.g., whether an advertisement as an action results in a click).
% If some context variables in $X$ are discrete, our analysis still holds conditional on the discrete variables.
Let a tuple of $m$ $\mathbb R$-valued random variables $(Y(1), \dots, Y(m))$ denote \emph{potential rewards}; $Y(a)$ denotes a \emph{potential reward} that is observed when action $a$ is chosen.
$(Y(1), \dots, Y(m), X)$ follows distribution $P$, which is unknown to the decision maker.

% We consider policies that choose actions based on individual context $X$.
A \emph{policy} chooses an action given a context.
Let $ML:\mathbb{R}^p\rightarrow \Delta({\cal A})$ represent the \textit{logging policy}, where $ML(a|x)$ is the probability of taking action $a$ for individuals with context $x$.
We assume that the analyst knows the logging policy and is able to simulate it.
That is, the analyst is able to compute the probability $ML(a|x)$ for each action $a\in{\cal A}$ given any context $x\in\mathbb{R}^p$.
%In typical machine-learning scenarios, an algorithm first applies machine learning on $X$ to make some prediction and then uses the prediction to output the choice probability $ML(\cdot|X)$.
% We allow for the case with deterministic policies, in which $ML(a|x)\in\{0,1\}$ for every $(a,x)$.
%We also allow for the case with fully stochastic algorithms in which $ML(a|x)>0$ for all $a\in{\cal A}$ and $x\in\mathbb{R}^p$.
%In addition, the algorithm can be mixed stochastic and deterministic in that $ML(a|x)>0$ for all $a\in{\cal A}$ for some $x\in\mathbb{R}^p$ but $ML(a|x)=1$ for some $a\in{\cal A}$ and $x\in\mathbb{R}^p$.
Suppose we have log data $\{(Y_i,X_i,A_i)\}_{i=1}^n$ generated as follows. For each individual $i$, (1) $(Y_i(1), \dots, Y_i(m),X_i)$ is i.i.d.~drawn from $P$;
\footnote{
This assumption is valid when we have a batch of log data generated by a fixed policy.
}
(2) Given $X_i$, the action $A_i$ is randomly chosen based on the probability $ML(\cdot|X_i)$;
% so that $A_i$ is independent of $Y_i(\cdot)$ conditional on $X_i$; %Note that we allow $ML(\cdot|X_i)$ to be deterministic, so that $ML(a|X_i) = 1$ for some action $a$ and the probability of selecting any other action is 0.
(3) We observe the reward $Y_i \coloneqq Y_i(A_i)$.
%, which we denote by $Y_i$.
Note that only one of $Y_i(1), \dots, Y_i(m)$ is observed for individual $i$ and recorded as $Y_i$ in the log data.
The joint distribution of $(Y, X, A)$ is determined once $ML$ and $P$ are given.

\begin{comment}
Assumption \ref{assumption:ML} \ref{assumption:ML-continuity} allows the function $ML$ to be discontinuous on a set of points with the Lebesgue measure zero.
For example, $ML$ is allowed to be a discontinuous step function as long as it is continuous almost everywhere.
Assumption \ref{assumption:ML} \ref{assumption:ML-measure-zero-boundary} holds if the Lebesgue measures of the boundaries of ${\cal X}_0$ and ${\cal X}_1$ are zero.
\end{comment}

{\bf Prediction Target.}
We are interested in estimating the expected reward from any given {\it counterfactual policy} $\pi:\mathbb{R}^p\rightarrow \Delta({\cal A})$, which chooses a distribution of actions given individual context:
\vspace{-3pt}
\begin{align*}
% \label{eq:V^pi}
	V(\pi)&\coloneqq E \left[\sum_{a\in {\cal A}}Y(a)\pi(a|X) \right].
\end{align*}

\section{Learning with Infinite Data}\label{section:identification}

We first consider the identification problem, which asks whether it is possible to learn $V(\pi)$ if we had an infinite amount of data.
Formally, we say that $V(\pi)$ is \textit{identified} if it is uniquely determined by the joint distribution of $(Y,X,A)$.
A key step toward answering the identification question is what we call the \textit{Approximate Propensity Score} (APS). To define it, for $a\in{\cal A}$ and $x\in{\cal X}$, let:
\begin{align*}
	p^{ML}_\delta(a|x)&\coloneqq\frac{\int_{B(x,\delta)}ML(a|x^*)dx^*}{\int_{B(x,\delta)}dx^*},
\end{align*}
where $B(x, \delta)=\{x^*\in\mathbb{R}^p:\|x-x^*\|<\delta\}$ is the $\delta$-ball around $x\in {\cal X}$.
Here, $\|\cdot\|$ denotes the Euclidean norm on $\mathbb{R}^p$.
To make common $\delta$ for all dimensions reasonable, we normalize $X_{ij}$ to have mean zero and variance one for each $j=1,...,p$.
We assume that $ML$ is a Lebesgue measurable function so that the integrals exist.
We then define APS $p^{ML}$ as follows: for $a\in{\cal A}$ and $x\in{\cal X}$,
\begin{align*}
	p^{ML}(a|x)&\coloneqq \lim_{\delta\rightarrow 0}p^{ML}_\delta(a|x).
\end{align*}
%Intuitively, $p^{ML}(a|x)$ is the average probability of choosing action $a$ in a shrinking ball around $x$.
%We call this the {\it Approximate} Propensity Score, since this score modifies the standard propensity score $ML(\cdot|X)$ to incorporate local variation in the score.

%TODO: Decrease the font size in the figure
\begin{figure}[!t]
\begin{minipage}[t]{.48\textwidth}
\centering
	\includegraphics[width=\linewidth]{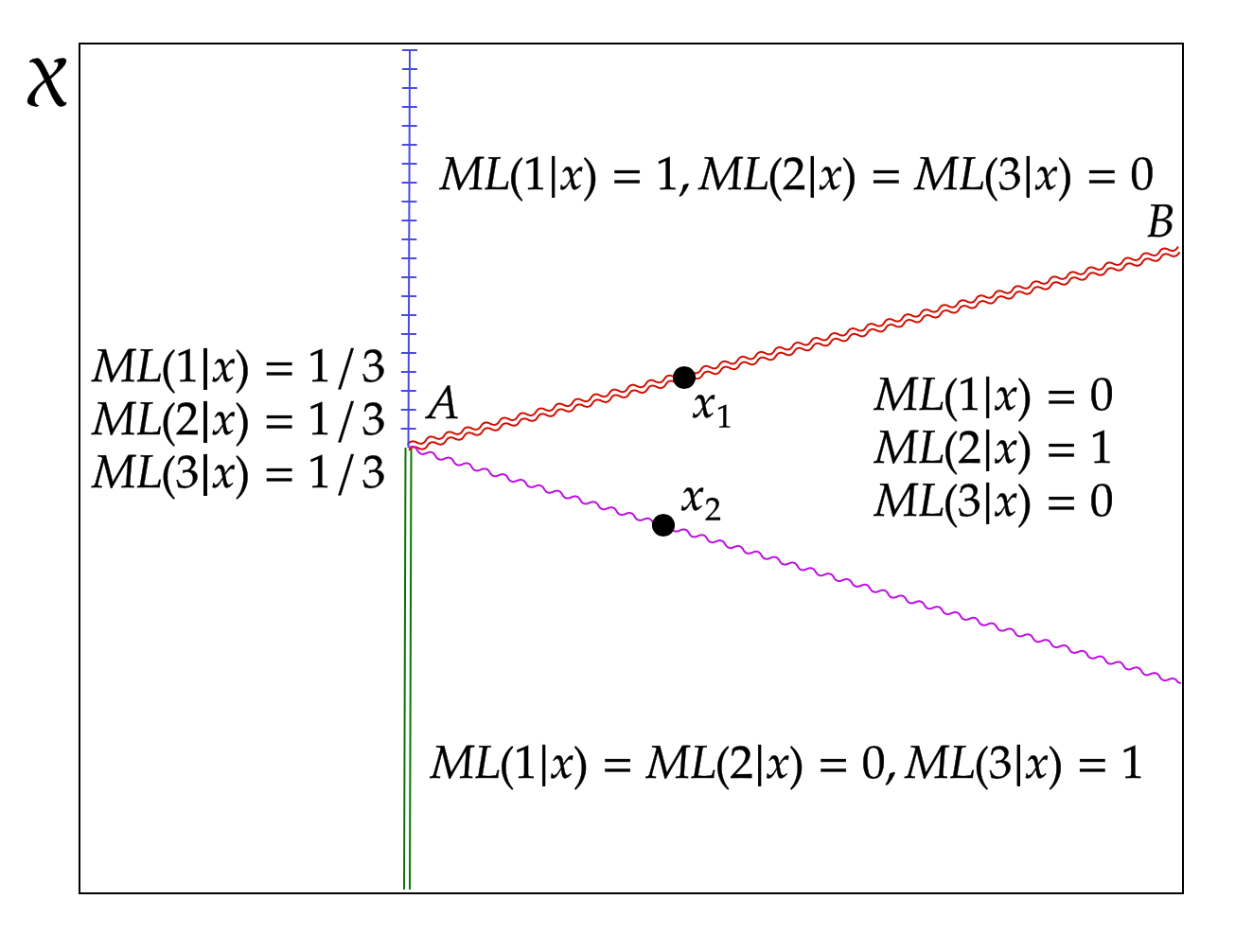}
% 	\subcaption{}
(a)  

\end{minipage}
\hfill
\begin{minipage}[t]{.48\textwidth}
\centering
	\includegraphics[width=\linewidth]{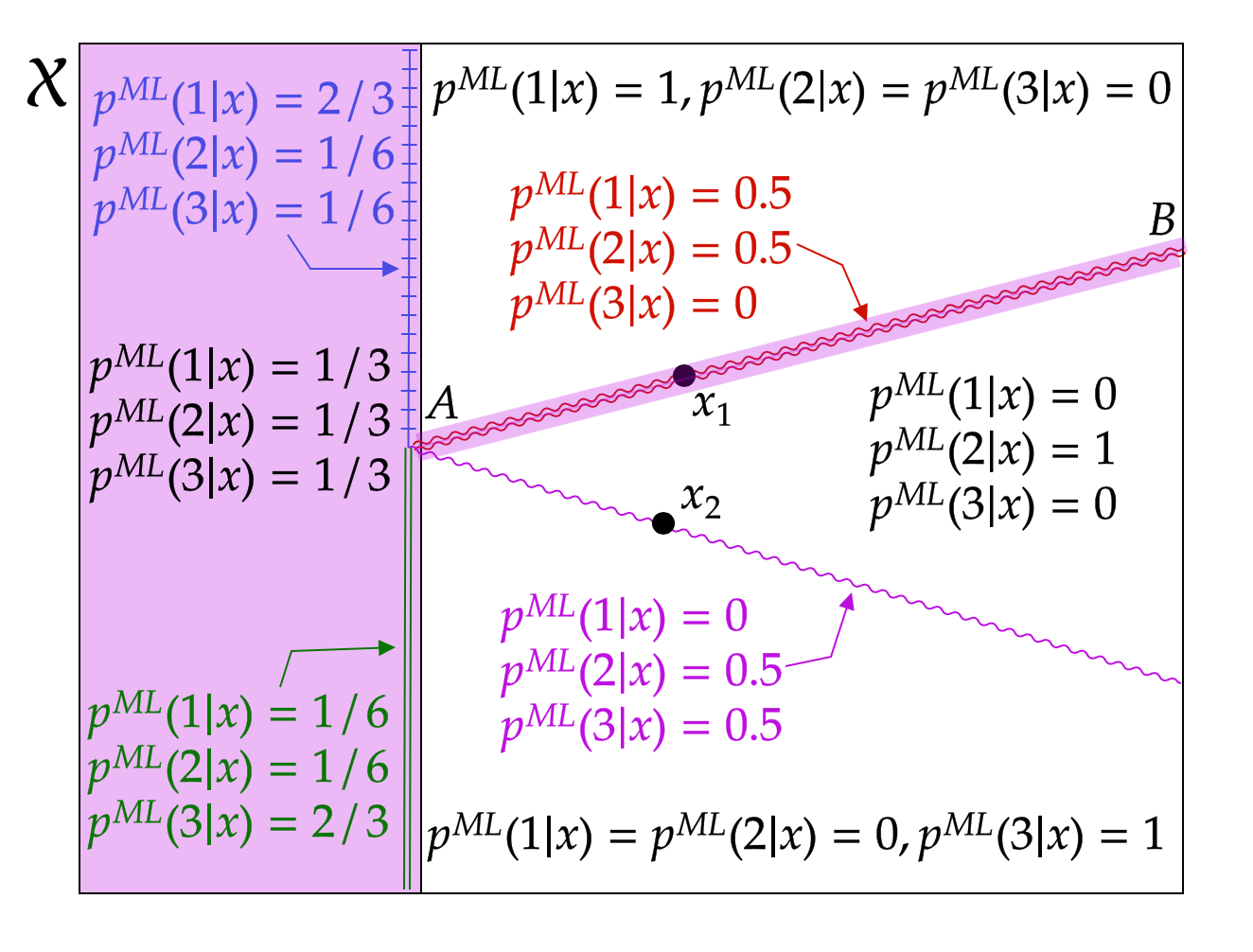}
% 	\subcaption[]{(b)}
(b)
\end{minipage}
\caption{Example of the Approximate Propensity Score}
\begin{flushleft}
	\footnotesize {\it Notes}: This figure shows an example of logging policy $ML$ (panel (a)) and corresponding APS $p^{ML}$ (panel (b)).
	The shaded region in panel (b) indicates the subpopulation for which $p^{ML}(1|x)>0$ and $p^{ML}(2|x)>0$.
	As discussed in Section \ref{section:estimation}, our method uses the subsample in the shaded region to estimate the conditional mean difference $E[Y(2)|X]-E[Y(1)|X]$.
\end{flushleft}
\label{figure:APS-example}
\end{figure}
%Let $\Delta_{int}({\cal A})\coloneqq \{p \in \Delta({\cal A}):p(a)\in (0,1) \}$.
% and $\Delta_{bd}^a({\cal A})\coloneqq \{p \in \Delta({\cal A}):p(a)\in \{0,1\} \}$ for each $a\in {\cal A}$.
%$\Delta_{int}({\cal A})$ is the set of probability vectors that assign every action a nonzero probability.
%, and $\Delta_{bd}^a({\cal A})$ is the set of probability vectors that assign action $a$ probability zero or one.

Figure \ref{figure:APS-example} illustrates APS. Here $\mathcal X \subseteq \mathbb{R}^2$, $\mathcal{A}=\{1,2,3\}$, and the support of $X$ is divided into four sets depending on the value of $ML$ as in panel (a).
Panel (b) shows the corresponding APS.
For the interior points of each of the four sets, APS is equal to $ML$.
On the border of any two sets, APS is the average of the $ML$ values in the two sets.

Our identification analysis uses the following assumption.

\begin{assumption}[Local Mean Continuity]\label{continuity}
	For any $a\in {\cal A}$, the conditional expectation function $E[Y(a)|X=x]$ is continuous at each $x\in {\cal X}$ such that $p^{ML}(a|x)>0$ and $ML(a|x)=0$.
\end{assumption}
$ML(a|x)=0$ means that action $a$ is never taken for individuals with context $x$.
If APS of $a$ at $x$ is nonzero ($p^{ML}(a|x)>0$), however, there exists a point close to $x$ that has a positive probability of receiving action $a$, which enables us to observe the reward from the action near $x$.
For any such point $x$, Assumption \ref{continuity} ensures that the points close to $x$ have similar conditional means of the potential reward $Y(a)$.
%Under the above assumption, APS provides an easy-to-check condition for whetherthe data from a logging policy allow us to identify the conditional mean potential rewards.
Thus, the conditional mean reward from action $a$ at $x$ is identified.
On the other hand, when $ML(a|x)>0$, action-context pair $(a,x)$ is observed, allowing us to identify the mean reward without any assumptions. Assumption \ref{continuity} therefore does not impose continuity at such points.
The lemma below summarizes the above argument.
For a set $A\subset\mathbb{R}^p$, let ${\rm int}(A)$ denote the interior of $A$.
\begin{lemma}[Identification of Conditional Means]\label{proposition:id-means}
	If Assumption \ref{continuity} holds, then for each $a\in {\cal A}$,
	$E[Y(a)| X=x]$ is identified for every $x\in {\rm int}({\cal X})$ such that $p^{ML}(a|x)> 0$.
\end{lemma}

%Lemma \ref{proposition:id-means} presents a sufficient condition for identification of the conditional mean potential rewards.
%Nonzero APS of $a$ at $x$ implies that there are individuals who receive action $a$ among those whose $X$ is close to $x$.
%Assumption \ref{continuity} ensures that those individuals are similar in terms of the mean potential reward.
%We can therefore identify the conditional mean potential reward from action $a$ at $x$.

We use Lemma \ref{proposition:id-means} to analyse identification of $V(\pi)$.
Suppose first that $\pi(a|x)>0 \Longrightarrow p^{ML}(a|x)> 0$, that is, the counterfactual policy $\pi$ only chooses actions with nonzero APS.
Lemma \ref{proposition:id-means} implies that the conditional mean reward is identified at every $(a,x)$ pair that could be realized under the policy $\pi$.
As a result, the expected reward $V(\pi)$ is identified for any such policy.
However, if there exists $(a,x)$ such that $\pi(a|x)>0$ but $p^{ML}(a|x)=0$, we cannot identify $V(\pi)$ without additional assumptions.
To be able to identify $V(\pi)$ for any policy $\pi$, we assume that the difference in the conditional mean reward function $E[Y(a)|X]$ between any two actions is constant over ${\cal X}$.
%, while allowing $E[Y(a)|X]$ itself to vary arbitrarily.

\begin{assumption}[Constant Conditional Mean Differences]\label{constant}
	There exists a function $\beta:{\cal A}\times{\cal A}\rightarrow \mathbb{R}$ such that $E[Y(a)|X]-E[Y(a')|X]=\beta(a,a')$.
\end{assumption}

%Note that Assumption \ref{constant} places restrictions only on the difference $E[Y(a)|X]-E[Y(a')|X]$.
%The conditional mean $E[Y(a)|X]$ can vary arbitrarily as $x$ changes.

At the end of Section \ref{section:estimation}, we discuss how our results would change if we drop Assumption \ref{constant} and a potential way of relaxing this.
We also impose the following condition on APS.

\begin{assumption}[Existence of Nonzero APS]\label{path-existence}
	For every $a\in \{2,...,m\}$, there exists a sequence $\{a_1,...,a_L\}$ with $a_1=1$ and $a_L=a$ for which the following condition holds: for every $l\in \{1,...,L-1\}$, there exists $x\in{\rm int}({\cal X})$ such that $p^{ML}(a_l|x)>0$ and $p^{ML}(a_{l+1}|x)>0$.
\end{assumption}

\begin{comment}
\begin{assumption}\label{permutation}
	There exists a permutation $\sigma$ of ${\cal A}$ such that for every $k\in \{1,...,m-1\}$, there exists $x\in{\rm int}({\cal X})$ such that $p^{ML}(\sigma(k)|x)>0$ and $p^{ML}(\sigma(k+1)|x)>0$.
\end{assumption}
\end{comment}

Assumption \ref{path-existence} states that there exists a path from a baseline action ($a_1=1$) to any other action ($a_L=a$) for which APS of any two consecutive actions ($a_l$ and $a_{l+1}$) is positive at some $x$.
For example, suppose that $m=3$, $p^{ML}(1|x_1)>0$, $p^{ML}(2|x_1)>0$, $p^{ML}(2|x_2)>0$ and $p^{ML}(3|x_2)>0$ for some $x_1,x_2\in {\cal X}$ as in Figure \ref{figure:APS-example} (b).
In this case, the sequence $\{1,2\}$ satisfies the condition in Assumption \ref{path-existence} for $a=2$, and the sequence $\{1,2,3\}$ satisfies the condition for $a=3$.
By Lemma \ref{proposition:id-means}, the four conditional means $E[Y(1)|X=x_1]$, $E[Y(2)|X=x_1]$, $E[Y(2)|X=x_2]$ and $E[Y(3)|X=x_2]$ are identified. Hence, the two differences $E[Y(1)|X=x_1]-E[Y(2)|X=x_1]$ and $E[Y(2)|X=x_2]-E[Y(3)|X=x_2]$ are identified.
Under Assumption \ref{constant}, the two differences do not depend on $x$. As a result, $E[Y(1)|X=x]-E[Y(2)|X=x]$ and $E[Y(2)|X=x]-E[Y(3)|X=x]$ are identified for every $x\in{\cal X}$.
Noting that $E[Y(a)|X=x]$ is identified for at least one $a\in{\cal A}$ for every $x\in{\cal X}$, we can use the differences to identify $E[Y(a)|X=x]$ for every $(a,x)$ pair, even for those not observed in data.
Thus, $V(\pi)$ is identified for any policy $\pi$.

\begin{proposition}[Identification of $V(\pi)$]\label{proposition:id-value-every-policy}
	Under Assumptions \ref{continuity}--\ref{path-existence},  $V(\pi)$ is identified for any policy $\pi$.
\end{proposition}

Assumption \ref{path-existence} typically holds if every action is chosen with a positive probability in some region of the context space $\cal{X}$.
For example, consider a deterministic logging policy that chooses the action with the largest predicted conditional mean reward given the context ($E[Y(a)|X]$), where the predictions are obtained from supervised learning.
If every action $a$ has a region where it is predicted to be optimal, then every action usually shares boundaries with at least one other action. Since $p^{ML}(a|x)>0$ and $p^{ML}(a'|x)>0$ at the boundaries shared by two actions $a$ and $a'$ (unless the boundaries are irregularly shaped), we can find a sequence of actions that satisfies Assumption \ref{path-existence}.
%Note that such borders, or decision boundaries, naturally arise when actions are chosen by a heterogeneous treatment effect model (i.e., for each individual, the potential rewards are first estimated, and then the best action is chosen for him/her.)

\section{Learning with Finite Data}\label{section:estimation}

{\bf OPE Estimator.}
Suppose that we observe a sample $\{(Y_i,X_i,A_i)\}_{i=1}^n$ of size $n$.
We propose an OPE estimator based on the following expression of our prediction target $V(\pi)$: under Assumption \ref{constant},
\begin{align}
    V(\pi) &=V(ML) \nonumber\\
    &~~~~~+E \left[\sum_{a=2}^m\beta(a,1) \bigl(\pi(a|X)-ML(a|X) \bigr) \right].\label{eq:value-expression}
\end{align}
Appendix \ref{section:derivation} derives this expression.
Since $V(ML)$ is the value from the logging policy $ML$, $V(ML)$ can be estimated by the sample mean of $Y_i$.
Our identification analysis suggests a way of conducting OPE on any policy $\pi$: (1) estimate $\beta(a,a')$ for each $(a,a')$ pair such that $p^{ML}(a|x)>0$ and $p^{ML}(a'|x)>0$ for some $x$; (2) use the estimates to recover $\beta(a,1)$ for every $a\in\{2,...m\}$ and plug them into the sample analogue of the above expression.
For simplicity, we consider a setup in which $p^{ML}(a|x)>0$ and $p^{ML}(1|x)>0$ for some $x$ for every $a$ so that we can directly estimate $\beta(a,1)$ in step (1) above.
%\footnote{***KY: This simplifies the analysis in the following two ways: (1) we don't need to consider how to infer whether $p^{ML}(a|x)>0$ and $p^{ML}(1|x)>0$ (or whether there are many such points). (2) we don't need to consider how to aggregate different estimates of $\beta(a,1)$ if we can estimate, for example, not only $\beta(a,1)$ but also $\beta(a,a')$ and $\beta(a',1)$. Another reason to make this simplification is that the current simulation uses action 1 as baseline and directly estimates $\beta(a,1)$.***}

To estimate $\beta(a,1)$, we use the subsample
\[
{\cal I}(a;\delta_n)\coloneqq
\left\{i: A_i\in\{1,a\}, q_{\delta_n}^{ML}(a \mid X_i)\in (0,1)
\right\},
\]
where
%\begin{equation}\label{APS-two-actions}
\[
q_{\delta_n}^{ML}(a \mid X_i) \coloneqq  \frac{p_{\delta_n}^{ML}(a \mid X_i)}{p_{\delta_n}^{ML}(a \mid X_i) + p_{\delta_n}^{ML}(1 \mid X_i)},
\]
%\end{equation}
and $\delta_n$ is a given bandwidth. The bandwidth shrinks towards zero as the sample size $n$ increases.\footnote{For the bandwidth $\delta_n$, we suggest considering several different values and check if the estimates are robust to bandwidth changes. It is hard to pick $\delta_n$ in a data-driven way to minimize the mean squared error, since it would require nonparametric estimation of functions on the high-dimensional context space.% such as conditional mean reward functions.
}
$q_{\delta_n}^{ML}(a| X_i)$ can be viewed as APS of action $a$ within the subsample for which either action $1$ or $a$ is assigned.
%for which either action $1$ or $a$ is assigned ($A_i\in\{1,a\}$), $p_{\delta_n}^{ML}(1| X_i)>0$, and $p_{\delta_n}^{ML}(a| X_i)>0$.
The subsample ${\cal I}(a;\delta_n)$ contains all observations $i$ such that both actions $1$ and $a$ can be chosen by the logging policy locally around $X_i$.
For example, in Figure \ref{figure:APS-example} (b), the shaded region corresponds to the subsample ${\cal I}(2;\delta_n)$.
This covers not only the subsample subject to full randomization (for which $ML(1|x)=ML(2|x)=ML(3|x)=1/3$) but also the local subsample near the deterministic decision boundary $AB$ between actions $1$ and $2$.

We propose minimizing the sum of squared errors on the subsample ${\cal I}(a;\delta_n)$:
%\footnotesize
\begin{align}
&(\hat\alpha_a,\hat\beta_a,\hat\gamma_a) =\argmin_{(\alpha_a,\beta_a,\gamma_a)}\nonumber \\ 
&\sum_{i\in {\cal I}(a;\delta_n)}
\Bigl(Y_i-\alpha_a - \beta_{a} 1\{A_i = a\}
% \nonumber\\ &\hspace{13em}
- \gamma_{a} q_{\delta_n}^{ML}(a| X_i) \Bigr)^2,  \label{olstrue}
\end{align}
%\normalsize
% \begin{equation}\label{olstrue}
% (\hat\alpha_a,\hat\beta_a,\hat\gamma_a)= \argmin_{(\alpha_a,\beta_a,\gamma_a)}\sum_{i\in {\cal I}(a;\delta_n)}
% \left(Y_i-\alpha_a - \beta_{a} 1\{A_i = a\} - \gamma_{a} q_{\delta_n}^{ML}(a| X_i) \right)^2,
% %	Y_i = \alpha_a + \beta_{a} 1\{A_i = a\} + \gamma_{a} q_{\delta_n}^{ML}(a| X_i) + \epsilon_{i,a},
% \end{equation}
where $1\{\cdot\}$ is the indicator function.
$\hat\beta_a$ is our estimator of $\beta(a,1)$.
We include $q_{\delta_n}^{ML}(a| X_i)$ as an explanatory variable to adjust for imbalance in the context distribution between actions $1$ and $a$, %those with $A_i=1$ and those with $A_i=a$,
as is done with the standard propensity score \citep{angrist_mostly_2008,hull}.
%$$
%q_{\delta_n}^{ML}(a| X_i) =  \frac{p_{\delta_n}^{ML}(a|X_i)}{p_{\delta_n}^{ML}(a|X_i) + p_{\delta_n}^{ML}(1|X_i)},
%$$
%and bandwidth $\delta_n$ shrinks towards zero as the sample size $n$ increases.
%$q_{\delta_n}^{ML}(a| X_i)$ can be viewed as APS of action $a$ within the subsample, and we include this as an explanatory variable to adjust for covariate imbalance as is typically done with the standard propensity score (\textcolor{red}{***CITATION HERE***}).
%Let $\mathbf{Z}_{i,n}=(1,1\{A_i=a\},q_{\delta_n}^{ML}(a|X_i))'$, and $I_{i,n}=1\{q_{\delta_n}^{ML}(a|X_i)\in (0,1)\}$.
%The LR estimator $(\hat\alpha_a,\hat\beta_a,\hat\gamma_a')'$ is then given by
%$$
%\begin{pmatrix}\hat\alpha_a\\\hat\beta_a\\\hat\gamma_a\end{pmatrix}=(\sum_{i: A_i\in \{1,a\}}\mathbf{Z}_i\mathbf{Z}_{i,n}'I_{i,n})^{-1}\sum_{i:A_i\in \{1,a\}}\mathbf{Z}_{i,n}Y_iI_{i,n}.
%$$
%The second element $\hat\beta_a$ is our estimator of $\beta(a,1)$.
We then define our OPE estimator as:
\begin{align}
	&\hat{V}(\pi) \nonumber\\
	&= \frac{1}{n}\sum_{i = 1}^{n} \left(Y_i + \sum_{a=2}^m \hat{\beta}_{a} \bigl( \pi(a|X_i) - ML(a|X_i) \bigr) \right).\label{ope}
\end{align}
%Consider the below ordinary least squares (OLS) regression:
%\begin{equation}
%	Y_i = \alpha + \sum_{a \in \cal A} \beta_a 1\{A_i = a\} + \sum_{a \in \cal A} \gamma_a p_{\delta_n}^{ML}(a|X_i) + \epsilon_i
%\end{equation}
%where bandwidth $\delta_n$ shrinks towards zero as the sample size n increases. Let $\hat{\beta}_a$ denote the OLS estimators of $\beta_a$ in the above regression.
It is worth noting that our method does not require the model selection.

For estimating $\beta(a,1)$, the above method uses APS $p_{\delta_n}^{ML}(a|X_i)$, which may be difficult to compute analytically if $ML$ is complex. In such a case, we propose approximating it by brute force simulation. We draw a value of $x$ from the uniform distribution on $B(X_i, \delta_n)$ a number of times, compute $ML(a|x)$ for each draw, and take the average of $ML(a|x)$ over the draws.\footnote{%Formally, let $X_{i,1}^*, \ldots, X_{i,S}^*$ be $S$ independent draws from the uniform distribution on $B(X_i, \delta_n)$, and calculate
%$p_{\delta_n}^{s}(a|X_i) = \frac{1}{S} \sum_{s = 1}^{S} ML(a|X_{i,s}^*)$ for each $a = 1,\ldots,m$.
%For fixed $n$ and $X_i$,
The approximation error of the simulated APS relative to true $p_{\delta_n}^{ML}(a|X_i)$ has a
% $1/\sqrt{S}$
$S^{-\frac{1}{2}}$
rate of convergence, where $S$ is the number of simulation draws.
This rate does not depend on the dimension of $X_i$, so the simulation error can be made negligible by using a large number of simulation draws even when $X_i$ is high dimensional.
}
%We compute $p_{\delta_n}^{s}(a|X_i)$ for each $i = 1,\ldots, n$ independently.
We then use it instead of $p_{\delta_n}^{ML}(a|X_i)$ to compute $q_{\delta_n}^{ML}(a| X_i)$, and then compute $\hat\beta(a,1)$ and $\hat V(\pi)$ as in (\ref{olstrue}) and (\ref{ope}).

\begin{comment}
We then estimate $\beta(a,1)$ by the following simulation version of the regression in (\ref{olstrue}) using the subsample for which $A_i\in\{1,a\}$, $p_{\delta_n}^{s}(1| X_i)>0$ and $p_{\delta_n}^{s}(a| X_i)>0$.
\begin{equation}\label{olssim}
	Y_i = \alpha_a^s + \beta_{a}^s 1\{A_i = a\} + \gamma_{a}^s q_{\delta_n}^{s}(a | X_i) + \epsilon_{i,a}^s,
\end{equation}
where
$$
q_{\delta_n}^{s}(a| X_i) =  \frac{p_{\delta_n}^{s}(a|X_i)}{p_{\delta_n}^{s}(a|X_i) + p_{\delta_n}^{s}(1|X_i)}.
$$
The estimator $\hat\beta_{a}^s$ converges to $\hat\beta_a$ as the number of simulation draws increases.
\end{comment}
%\begin{equation}\label{olssim}
%	Y_i = \alpha + \sum_{a \in \cal A} \beta_a 1\{A_i = a\} + \sum_{a \in \cal A} \gamma_a p_{\delta_n}^{s}(a|X_i) + \epsilon_i
%\end{equation}
%Let $\hat{\beta}_{a}^s$ denote the OLS estimator of the simulation regression.

{\bf Consistency.}
We show that $\hat{V}(\pi)$ is a consistent estimator of $V(\pi)$, that is, $\hat{V}(\pi)$ converges in probability to $V(\pi)$ as $n\rightarrow \infty$ under some regularity conditions.
%To do so, we first prove that $\hat{\beta}_{a}$ is consistent for $\beta(a,1)$ for all $a\in\{2,...,m\}$.

\begin{assumption}[Regularity conditions]\label{consassump}
See Appendix~\ref{appendix:discuss_assump} for details.
\end{assumption}

	\begin{comment}

	Part \ref{assumption:p-1dim} of Assumption \ref{consassump} \ref{assumption:boundary-measure} says that $\partial\Omega^*_a$ is $(p-1)$ dimensional and has nonzero density.
	Part \ref{assumption:zero-set} requires that the logging policy choose either action $1$ or $a$ near the boundary of $\Omega^*_a$ even if the context value is not in the subsample ${\cal X}_{a,1}$ as long as it is in the neighborhood of ${\cal X}_{a,1}$.
	%Part \ref{assumption:zero-set} puts a restriction on the values $ML$ takes on outside the support of $X_i$.
	%It requires that $ML(x)=0$ for almost every $x\notin\Omega^*$ that is outside ${\cal X}$ but is in the neighborhood of ${\cal X}$.
	%$ML(x)$ may take on any value if $x$ is not close to ${\cal X}$.
	%These conditions hold in practice.

	Lastly, Assumption \ref{consassump} \ref{assumption:boundary-continuity} imposes continuity and boundedness on the conditional moments of rewards and the probability density near the boundary of $\Omega^*_a$.

\begin{theorem}[Consistency of $\hat\beta_a$]\label{thm:olsconsistency}
	Suppose that Assumptions \ref{constant} and \ref{consassump} hold, $\delta_n \rightarrow 0$, and $n\delta_n \rightarrow \infty$ as $n \rightarrow \infty$. Then $\hat{\beta}_{a}$ converge in probability to $\beta(a,1)$
    for all $a\in\{2,...,m\}$.
\end{theorem}

This result leads us directly to the primary consistency result for the OPE estimator $\hat V(\pi)$.
\end{comment}
\begin{theorem}[Consistency of $\hat V(\pi)$]\label{cor:opeconsistency}
	Suppose that Assumptions \ref{constant} and \ref{consassump} hold, $\delta_n \rightarrow 0$, and $n\delta_n \rightarrow \infty$ as $n \rightarrow \infty$. Then $\hat{V}(\pi)$ converges in probability to $V(\pi)$ for every policy $\pi$.
\end{theorem}

The main argument in the proof of Theorem 1 is similar to the one used for the consistency result of \citet{narita2021algorithm} (the first part of their Theorem 1). We extend their result to OPE with multiple actions.

Our consistency result requires that $\delta_n$ go to zero slower than $n^{-1}$.
This ensures that, when $ML$ is deterministic, we have sufficiently many observations in the $\delta_n$-neighborhood of the boundary of $\Omega^*_a \coloneqq \{x \colon ML(a | x) = 1\}$ (the set of the context values for which the probability of choosing action $a$ is one).
% for which $q_{\delta_n}^{ML}(a| X_i)\in (0,1)$.
Importantly, the rate condition does not depend on the dimension of $X_i$.
This is because we use all the observations in the $\delta_n$-neighborhood of the boundary, and the number of those observations is of order $n\delta_n$ regardless of the dimension of $X_i$ if the boundary is $(p-1)$ dimensional.
%the dimension of the boundary is one less than the dimension of $X_i$, i.e., $(p-1)$.
Our estimator is therefore expected to perform well even if $X_i$ is high dimensional.

Our result holds under the assumption of constant conditional mean reward differences (Assumption \ref{constant}). If this assumption does not hold for a deterministic logging policy, $\hat\beta_a$ is a consistent estimator of the mean reward difference for the subpopulation on the decision boundary between actions $a$ and $1$ (see Appendix \ref{proof:opeconsistency}). Therefore, our estimator may still perform well when we are interested in a counterfactual policy that marginally changes the logging policy's decision boundary. 

One way to relax Assumption \ref{constant} is to consider a partition of ${\cal X}$ and assume that the conditional mean difference between any two actions is constant within each cell in the partition.
This allows the conditional mean differences to vary across cells.
If for each $(a,a')$ pair, each cell contains $x$ such that $p^{ML}(a|x)>0$ and $p^{ML}(a'|x)>0$, we can consistently estimate the conditional mean differences and the expected reward from any policy.
How to find such a partition is an interesting future topic.

\section{Simulations}\label{section:sim}

%We conduct two Monte Carlo experiments to assess the feasibility and performance of our method.
\subsection{Experiment 1: Mix of A/B Test and Deterministic Logging Policy}\label{section:exp-mix}
Consider a tech company that conducts an A/B test using a small segment of the population. The company applies a deterministic logging policy to the rest of the population.
We generate a random sample $\{(Y_i, X_i, A_i)\}_{i=1}^n$ of size $n=$ 50,000 as follows. There are 5 actions ($m=5$) and $100$ context variables ($p=100$), with $X_i \sim N(0, \boldsymbol\Sigma)$.
$Y_i(a)$ is generated as $Y_i(a) = 0.75 \sum_{k=1}^{100}X_{ki}^2\alpha_{0,k} + 0.25 u_i +\epsilon_i(a)$, where $\alpha_0=(\alpha_{0,1},...,\alpha_{0,100})\in \mathbb{R}^{100}$, $u_i \sim N(0,1)$, and $\epsilon_{i}(a) \sim N(a,1)$.
The conditional mean difference $E[Y_i(a)|X_i]-E[Y_i(1)|X_i]$ is constant over $x$.
The choice of parameters $\boldsymbol\Sigma$ and $\alpha_0$ is explained in Appendix \ref{appendix:sim_details}.
%and $A^{\circ 2}$ indicates the matrix $A$ with every element raised to the power of 2.
%This can be considered the baseline outcome when no action is chosen. $A_i$ is the choice of one of 5 actions for unit $i$, i.e. $A_i \in [5]$. We consider two models of potential outcomes for $Y_i(a)$, one in which Assumption \ref{constant} holds and another in which outcome differences vary by $X_i$.
%\begin{enumerate}[label=\Alph*.]
%	\item (Constant Conditional Mean Differences) $Y_{i}(a) = Y_i(0) + \epsilon_{ai}$
%	\item (Non-Constant Conditional Mean Differences) $Y_{i}(a) = Y_i(0) + X^{'\circ 2}_i\alpha_{a1}$
%\end{enumerate}
%where $\epsilon_{ai} \sim N(a,1)$.
%We consider two different sets of treatment and counterfactual algorithms -- mixed algorithms with a deterministic and stochastic component, and purely deterministic algorithms.
To generate $A_i$, let $q^k_{0.99}$ be the 99th percentile of the $k$th context variable $X_{ki}$. Let $\tau_{pred}^{ML}(x,a)$ be a prediction of the reward from action $a$ given context value $x$ obtained by supervised learning from a past, independent training sample $\tilde{\mathcal{D}}=\{(\tilde Y_i, \tilde X_i, \tilde A_i)\}_{i=1}^{\tilde n}$ of size $\tilde n =$ 10,000 (see Appendix \ref{appendix:sim_details} for how we constructed $\tilde{\mathcal{D}}$ and $\tau_{pred}^{ML}$).
$A_i$ is then generated based on the logging policy:
%For the mixed algorithms case, the function $ML$ is given by
\footnotesize
\[
ML(a|x) = \begin{cases}
	1/5 & \text{if } x_1  \geq q^1_{0.99} \\
	 1\left\{a=\argmax_{a' \in \{1,...,5\}} \tau_{pred}^{ML}(x, a') \right\} & \text{if } x_1 < q^1_{0.99}.
\end{cases}
\]
\normalsize
The first case corresponds to the A/B test segment while the second case to the deterministic policy segment.
Finally, $Y_i$ is generated as $Y_i=Y_i(A_i)$.

We simulate 1,000 hypothetical samples from the above data-generating process.
For each simulation, we use the simulated sample to estimate the value of a counterfactual policy $\pi$, another mix of an A/B test and a deterministic policy. With another reward prediction function $\tau_{pred}^{\pi}$,
\footnotesize
\[
\pi(a|x) = \begin{cases}
	1/5 & \text{if } x_2  \geq q^2_{0.99} \\
	1\left\{a=\argmax_{a' \in \{1,...,5\}} \tau_{pred}^{\pi}(x, a') \right\} & \text{if } x_2 < q^2_{0.99}.
\end{cases}
\]
\normalsize

{\bf Alternative Methods.}
We compare our method with two alternative estimators.
The first uses the A/B test segment (for which $ML(a|X_i)=1/5$) while the second uses the full sample. The methods first compute the simple mean differences in reward $Y_i$ between actions $a\in\{2,...,5\}$ and $1$, and then plugs them into $\hat\beta_a$ of Eq. (\ref{ope}).
Both our method and the alternative estimator with the A/B test segment produce consistent estimators of the prediction target $V(\pi)$.
However, the alternative uses only the A/B test segment while our method additionally uses the local subsample near the decision boundary of the deterministic policy as we discussed in Section \ref{section:estimation}.

{\bf Result.} The first panel of Table \ref{tbl:sim_result} presents the bias, standard deviation (S.D.) and root mean squared error (RMSE) of our proposed estimators with several choices of $\delta$ and two alternative estimators.
%The alternative method uses the A/B test segment (for which $ML(a|X_i)=1/5$) or the full sample to compute the mean differences in reward $Y_i$ between actions $a\in\{2,...,5\}$ and the baseline $1$ and then plugs them into $\hat\beta_a$ of Eq. (\ref{ope}).
The alternative estimator using the full sample has a larger bias than the other two, since it does not control for the difference in the context distribution between actions.
Our proposed estimator outperforms the alternative estimator using the A/B test sample in terms of RMSE.
This suggests that exploiting both of the A/B test segment and the local subsample near the deterministic decision boundary can lead to better performance than using only the A/B test segment.

\begin{comment}
The counterfactual policy $\pi$ is given by
\[
\pi(a|X_i) = \begin{cases}
	1/5 & \text{if } X_{2i}  \geq q^2_{0.99} \\
	1 & \text{if } X_{2i} < q^2_{0.99} \text{ and } \argmax_{a \in \cal A} \tau_{est}^{\pi}(X_i, a) = a \\
	0 & \text{otherwise}
\end{cases}
\]
$q^1_{0.99}$ and $q^2_{0.99}$ are the 99th percentile of $X_1$ and $X_2$ in the data, respectively. In other words, the decision rules are deterministic for all individuals outside of the top 1\% of $X_1$ and $X_2$. We explain how $\Sigma$, $\alpha_0$, $\alpha_{a1}$, and $\tau_{est}$ are generated in \ref{simdgp}.

$A_i$ in this case is generated as
\[
	A_i = \begin{cases}
		A^*_i & \text{if } X_{1i}  \geq q_{0.99} \\
		\argmax_{a \in \cal A} \tau_{est}^{ML}(X_i, a) & \text{if } X_{1i} < q_{0.99}
	\end{cases}
\]
where $A^*_i$ is a discrete random variable with $P(A^*_i = a) = 1/5$ for all $a \in [5]$.
\end{comment}

\begin{table*}[!t]
	\centering
% 	\small
    \tiny
	\caption{Simulation results: bias, S.D., and RMSE of estimators of $V(\pi)$ }
	\label{tbl:sim_result}
	\begin{threeparttable}
	\resizebox{0.88\textwidth}{!}{\begin{tabular}{lccccccc}
		\toprule \\ [-2.8ex]
		& \multicolumn{4}{c}{Our Proposed Method with APS Controls} & \multicolumn{2}{c}{Method with Mean Differences} & Direct \\
		\cline{2-5} \cline{6-7} \\[-2.2ex]
		 & $\delta = 0.1$ & $\delta=0.5$ & $\delta = 1$  & $\delta = 2.5$ &  A/B Test Sample &  Full Sample &  Method \\[-0.2ex]
		 & (1) & (2) & (3) & (4) & (5) & (6) & (7) \\
		\hline \\ [-2.0ex]
		\multicolumn{8}{c}{Experiment 1: Mix of A/B Test and Deterministic Logging Policy} \\ \\ [-1.8ex]
		%\multicolumn{5}{l}{Parameter value: $V(\pi) = 2.522$} \\ \\ [-1.0ex]
		 Bias & $-$.060 & $-$.057 & $-$.057 & $-$.060 & $-$.061 &  $-$.075 & --- \\ \\ [-1.8ex]
		  S.D. & .099 & .098 & .096 & .096 & .101 &  .103 & --- \\\\ [-1.8ex]
		  RMSE & .115 & .113 & .112 & .113 & .118 &  .128 & --- \\\\ [-1.0ex]
		  Avg. $N$  & 1862 & 6362 & 12502 & 33122 & 500  & 50000 & ---\\ \\ [-1.0ex]
		\multicolumn{8}{c}{Experiment 2: Upper Confidence Bound Logging Policy} \\ \\ [-1.8ex]
		%\multicolumn{5}{l}{Parameter value: $V(\pi) = 3.100$} \\ \\ [-1.0ex]
		 Bias & .048 & .047 & .046 & .047 & --- & --- & .342 \\ \\ [-1.8ex]
		  S.D. & .033 & .030 & .029 & .029 & --- & --- & .012 \\\\ [-1.8ex]
		  RMSE & .058 & .056 & .055 & .055 & --- & --- & .342 \\\\ [-1.0ex]
		  Avg. $N$  & 3397 & 17344 & 31107 & 47601 & --- & --- & 50000 \\
		\bottomrule
	\end{tabular}}
	\vspace{0.2em}
	\caption*{
	%\begin{flushleft}
	\footnotesize {\it Notes}: This table shows the bias, the standard deviation (S.D.), and the root mean squared error (RMSE) of the estimators of the reward from the counterfactual policy $V(\pi)$ in the two simulation experiments.
	We use $1,000$ simulations of a size $50,000$ sample to compute these statistics.
	Columns (1)--(4) report estimates from our method with several choices of $\delta$.
	Each APS is computed by averaging 100 simulation draws of the $ML$ value.
	In columns (5)--(6), we estimate the mean reward differences $\beta(a,1)$ by the sample mean differences in the A/B test segment and the full sample, respectively.
	In column (7), we estimate $\beta(a,1)$ by fitting a linear model that predicts the reward from the context and action.
	%These statistics are computed with the estimand set to the average reward $Y$ under the counterfactual algorithm assignment, in either the case where constant conditional mean differences hold or not.
	%We use $\delta\in\{0.1,0.5,1,2.5\}$ to compute APS.
	The bottom row of each panel shows the average number of observations with nonzero APS for every action (Columns (1)--(4)), that with nonzero $ML$ for every action (Column (5)), or the total sample size (Columns (6)--(7)).
% 	We used 100 15 GB CPU cores for computation.
%\end{flushleft}
}
\end{threeparttable}
\end{table*}

\subsection{Experiment 2: Upper Confidence Bound Logging Policy}
In the second experiment, both the logging policy and the counterfactual policy are deterministic. The rest of the setup is the same as that in the first experiment.
%For the deterministic algorithms case, we use two algorithms commonly applied in practice.
We first use the independent training sample $\tilde{\mathcal{D}}$ %Section \ref{section:exp-mix}
to train an Upper Confidence Bound bandit algorithm.
%Let $D_1(x_1)\in\{1,...,10\}$ indicate which decile of $X_{1i}$ the individual with $X_{1i}=x_1$ belongs to. Define $D_2(x_2)$ analogously for $X_{2i}$.
%Let $Q(a,d_1,d_2)$ be the sample mean reward for each action $a$ for every decile pair $(d_1,d_2)$ in the distribution of $X_{1i}$ and $X_{2i}$.
% in the sample $\tilde{\mathcal{D}}$.
The logging policy $ML$ is given by
%We freeze the algorithm after having it update on the independent sample, so that the counterfactual algorithm $ML$ is given by:
$
ML(a|x) =
	 1\left\{a=\arg\max_{a'\in\{1,...,5\}}UCB(x,a') \right\},
$
where $UCB(x,a)$ is an upper confidence bound of $E[Y_i(a)|X_i=x]$.
See Appendix \ref{appendix:sim_details} for training details.
%$UCB(x, a) = Q(a, D_1(x_1), D_2(x_2)) + c\sqrt{\frac{\log \tilde n}{\tilde N_{a, D_1(x_1), D_2(x_2)}}}$.
%Here, we set exploration parameter $c$ to $2$, $\tilde n(=10,000)$ is the size of the training sample $\tilde{\mathcal{D}}$. $\tilde N_{a, d_1, d_2}$ is the number of observations with action $a$ for the decile pair $(d_1,d_2)$ in the sample.
We do not update the policy while generating $\{(Y_i, X_i, A_i)\}_{i=1}^n$ in the simulation. The sample is a batch of log data.

For the counterfactual policy $\pi$, we use $\tilde{\mathcal{D}}$ to train a model $f(x,a)$ that predicts the reward given the context and action, using sklearn's RandomForestRegressor with 500 trees and otherwise default parameters.
The counterfactual policy tries to maximize the expected reward $V(\pi)$ by choosing the action with the largest predicted reward:
$
\pi(a|x) =
    1\left\{a=\arg\max_{a' \in\{1,...,5\}} f(x, a') \right\}.
$

{\bf Alternative Method.}
We compare our method with an alternative estimator using the Direct Method.
This first fits a linear model $Y_i=\alpha+\sum_{a = 2}^{5} \beta_{a} 1\left\{A_i = a\right\}+\sum_{k=1}^{100}X_{ki} \gamma_k + e_i$, then makes the reward prediction from action $a$ for individual $i$ by $\hat\mu_i(a)=Y_i+(\hat\beta_a-\hat\beta_{A_i})$, and finally computes $\hat V(\pi)=\frac{1}{n}\sum_{i=1}^n\sum_{a=1}^5\hat\mu_i(a)\pi(a|X_i)$.
%and plugs the estimated $\beta_a$ into Eq. (\ref{ope}).
The linear model used by this method correctly imposes the constant conditional mean differences but misspecifies the functional form with respect to $X_i$.

{\bf Result.} The second panel of Table \ref{tbl:sim_result} shows the result.
%In Column (7), we report the performance of an alternative estimator using the Direct Method.
The alternative using the Direct Method is significantly biased due to model misspecification.
%due to misspecification of the regression model.
Our proposed estimator seems to effectively use the local subsample near the decision boundary and has smaller bias and RMSE than the alternative.

% for AISTATS -- not the original position
\begin{table*}[!t]
% \small
\tiny
\centering
  \caption{Off-policy evaluation using policy's generated data}
  \label{table_ope_mercari}
  \begin{threeparttable}
  \resizebox{0.95\textwidth}{!}{
  \begin{tabular}{lcccccc}
  \toprule \\ [-2.8ex]
     & \multicolumn{5}{c}{Our Proposed Method with APS Controls} & Mean \\ \cline{2-6}\\[-2.0ex]
     & $\delta=0.4$ & $\delta=0.8$ & $\delta=1.2$ & $\delta=2.0$ & $\delta=3.0$ & Differences \\[-0.2ex]
     & (1) & (2) & (3) & (4) & (5) & (6)  \\ \hline \\ [-1.8ex]
    %Purchase & 1026 & 2387 & 2681 & 1558 & 2079 & -528\\
    %Value  & (1699) & (1120) & (876) & (806) & (621) & (0.093)\\ [1ex]

    %\# of transactions & 0.62 & 0.67 & 0.94 & 0.71 & 1.05 & -0.095 \\
    % & (0.72) & (0.49) & (0.40) & (0.35) & (0.27) & (0.506) \\ [1ex]

    %Coupon & 12.90 & 24.78 & 20.05 & 16.66 & 22.58 & 23.74 \\
    %Usage & (14.70) & (10.24) & (9.21) & (7.71) & (5.96) & (0.00) \\ [1ex]

    Effect on Purchase Value & 0.35 & 0.82 & 0.92 &  0.54 & 0.72 & $-$0.17\\
    & (0.59) & (0.39) & (0.30) & (0.28) & (0.21) & (0.11) \\ [1ex]

  Effect on \# of Transactions & 0.43 & 0.47 & 0.66 & 0.49 & 0.74  & $-$0.07\\\
    &  (0.50) & (0.34) & (0.28) & (0.25) & (0.19) & (0.10)\\ [1ex]

    Effect on Point Usage & 0.37 & 0.71 & 0.57 & 0.47 & 0.64 & 0.68\\
     & (0.42) & (0.29) & (0.26) & (0.22) & (0.17) & (0.04)\\ [1.5ex]

    Coupon Cost Effectiveness Measure & 79.57 & 96.35 & 134 & 93.51 & 92.07 & \textemdash\\
     & (130) & (48.97) & (61.97) & (49.33) &  (28.45) & \\ [1.5ex]
    $N$ & 2758 & 4688 & 6016 & 8085 & 9602 & 89486  \\
    %$N$ ($p^{ML}_{\delta}(2|X_i)\in (0,1)$) & 2758 & 4688 & 6016 & 8085 & 9602 & \textemdash  \\
    %$N$ (All) & \multicolumn{6}{c}{89486} \\
    \bottomrule
  \end{tabular}
  }
  \vspace{0.2em}
  \caption*{
  %\begin{flushleft}
  \footnotesize
  {\it Notes}: The first three rows of this table report estimated effects of the policy recommendation $A_i$ on purchase behavior.
  Columns (1)--(5) report estimates from our method with several choices of $\delta$ used to compute APS.
  Column (6) reports the outcome mean differences between those with $A_i=1$ and $A_i=0$.
  Each APS is computed by averaging $100$ simulation draws of the logging policy's binary decision.
  All numbers in the first three rows are normalized by dividing the original estimates by the sample outcome means.
  The fourth row reports our measure of coupon cost effectiveness, which predicts how much the purchase value would increase in USD if we increased the cost of the campaign by 1 USD.
  Heteroskedasticity-robust standard errors are reported in parentheses.
  The last row reports the number of observations with nonzero APS for every action (Columns (1)--(5)) or the total sample size (Column (6)).
  %\end{flushleft}
 }
  \end{threeparttable}
\end{table*}

% \newpage
\section{Real-World Application}

    {\bf Setup.} We apply our method to empirically evaluate a coupon targeting policy of an online platform.
    This application uses proprietary data provided by Mercari, Inc.
    % \footnote{We follow the Privacy Policy and Terms of Service of Mercari. The data do not contain any personally identifiable information after variable names are removed.\label{ftnt:privacy}}
    %, which we described in the introduction.
    %This company uses an algorithm based on an uplift model to determine to whom they should offer a coupon.
    This company conducts the following promotional campaign.
    They target customers who signed up for Mercari 4 days ago but have not made a purchase yet.
    The company uses a logging policy based on an uplift model to determine whether they offer a promotional coupon to each target customer.
    If customers receive the coupon and make a purchase, they get 900 points (equivalent to 8.34 USD) that they can use for future purchases.
    We observe data $(Y_i,X_i,A_i)$ for each target user $i$ from this campaign, where action $A_i\in\{0,1\}$ is whether the logging policy recommended offering the coupon to the customer ($A_i=1$) or not ($A_i=0$), $X_i$ is the vector of more than 200 input features for the uplift model, and $Y_i$ is an outcome such as the customer's spending after this coupon offer.

    %This algorithm produces logged data, which we use in this application.
    %To determine to whom they should offer this coupon, the company uses an algorithm based on an uplift model.
    %In the notation of our theoretical framework, action $A_i$ is whether the algorithm recommended offering the coupon to the customer, and $X_i$ is the vector of input features like the customer's past search behaviors.
    The company's logging policy works as follows.
    They first use data from a past A/B test and XGBoost to train a model of the conditional average effect of the coupon on purchases (they use library pylift for implementation).
    Let %$\tau:{\cal X}\rightarrow \mathbb{R}$ be the trained model, where
    $\tau(x)$ be the predicted coupon effect for those whose feature value is $X_i=x$.
    The logging policy then recommends offering a coupon to customer $i$ if the predicted effect is in the top 80\% of the distribution of predicted effects.
    That is, the logging policy $ML$ is given by
    $
    ML(1|x)=1\{\tau(x)\ge c\},
    $
    where $c$ is the 20th quantile of the distribution of $\tau(X_i)$.

    {\bf Effects of Policy Recommendation.} We first apply our method to the logged data generated by the above policy to estimate the effect of the policy recommendation $A_i$ ($\beta(1,0)=E[Y_i(1)-Y_i(0)]$) on the following three outcomes: (1) the purchase value (how much the customer spent), (2) the number of transactions, and (3) point usage (how many points the customer used).
    All outcomes are sums over 18 days after the coupon offer decision.
    We compute APS with $\delta\in\{0.4,0.8,1.2,2.0,3.0\}$.%by using $100$ simulation draws of the $ML$ value for each observation.
    \footnote{Unlike the theoretical framework, the feature vector $X_i$ consists of discrete and continuous variables.
    We compute APS by fixing the
    value of the discrete part and computing by simulation the APS integral with respect to the continuous part.  See Appendix \ref{section:discrete-X} for details.}

    Columns (1)--(5) in the first three rows of Table \ref{table_ope_mercari} report the estimated effects of the policy recommendation $A_i$.
    We normalize the estimates by dividing the original numbers by the sample outcome means for confidentiality.
    The results show that the effects of the policy recommendation $A_i$ on the purchase value, the number of transactions, and point usage are 35--92\%, 43--74\%, and 37--71\% of their sample means, respectively.
    These positive effects mark a sharp contrast with Column (6), which reports the simple differences in the outcome means between those with $A_i=1$ and those with $A_i=0$.
    The simple mean differences on the purchase value and the number of transactions are negative. These negative estimates suggest that the logging policy tends to recommend a coupon to the customers who have a low propensity to make purchases.
    Our proposed method corrects for this negative selection bias by controlling for APS.

    {\bf Evaluation of Counterfactual Policies.}
    The company needs to compensate for the discount that customers get by using points. Thus, adopting a new policy would be profitable only when the increase in revenue is sufficiently large compared to that in point usage. The company charges sellers 10\% of every payment from the buyer;
    % , which means
    the revenue increases by 10\% of the increase in purchase value. Hence, the policy change is beneficial if the ratio of the increases in the average purchase value and point usage is larger than 10.

    Suppose we change our policy from $ML$ to a counterfactual one $\pi$.
    Let $Y^1_i$ and $Y^2_i$ denote the purchase value and point usage respectively.
    Under the constant conditional effect assumption, i.e., $E[Y_i^1(1)-Y_i^1(0)|X_i]=:\beta$ and $E[Y_i^2(1)-Y_i^2(0)|X_i]=:\gamma$, the ratio is:
    % To evaluate how to improve the current policy, we next consider how much the purchase value would increase relative to cost if we switched from the status-quo to a new policy.
    % Let $Y^1_i$ and $Y^2_i$ denote the purchase value and point usage.
    % We view point usage as the cost of the campaign, since the company needs to compensate for the discount that customers get by using points.
    % Suppose we change our policy from $ML$ to a counterfactual one $\pi$.
    % If conditional effects are constant, that is, $E[Y_i^1(1)-Y_i^1(0)|X_i]=\beta$ and $E[Y_i^2(1)-Y_i^2(0)|X_i]=\gamma$, the ratio of the changes in the average purchase value and point usage are given by
        \begin{align*}
            %\frac{E[(Y_i^1(2)-Y_i^1(1))(ML'(2|X_i)-ML(2|X_i))]}{E[(Y^2_i(2)-Y^2_i(1))(ML'(2|X_i)-ML(2|X_i))]}=\frac{\beta E[(ML'(2|X_i)-ML(2|X_i))]}{\gamma E[(ML'(2|X_i)-ML(2|X_i))]} = \frac{\beta}{\gamma}.
            &\frac{E[\sum_{a=0}^1Y_i^1(a)\pi(a|X_i)]-E[\sum_{a=0}^1Y_i^1(a)ML(a|X_i)]}{E[\sum_{a=0}^1Y_i^2(a)\pi(a|X_i)]-E[\sum_{a=0}^1Y_i^2(a)ML(a|X_i)]} \\
            %&=\frac{E[E[Y_i^1(1)-Y_i^1(0)|X_i](\pi(1|X_i)-ML(1|X_i))]}{E[E[Y_i^2(1)-Y_i^2(0)|X_i](\pi(1|X_i)-ML(1|X_i))]}\\
            &=\frac{\beta E[\pi(1|X_i)-ML(1|X_i)]}{\gamma E[\pi(1|X_i)-ML(1|X_i)]} = \frac{\beta}{\gamma}.
        \end{align*}
    % The ratio $\beta/\gamma$ measures how much the purchase value would increase in USD if we increased the cost of the campaign by 1 USD by altering the policy to offer the coupon to more customers.

    The fourth row of Table \ref{table_ope_mercari} reports the estimates of the ratio $\beta/\gamma$.
    The estimates are larger than 10 for all $\delta$'s.
    % The company charges sellers 10\% of every payment from the buyer. The estimates therefore predict that if the cost increased by 1 USD, the revenue would increase by 8--13.4 USD.
    This result
    % confirms the effectiveness of the promotional offer and
    suggests that it would be profitable to expand the campaign.
    % by redesigning the policy.

    As mentioned in Section \ref{section:estimation}, without the constant conditional effect assumption, our estimator for the effect of the policy recommendation is a consistent estimator of the conditional effect for the subpopulation on the decision boundary, i.e., $E[Y_i(1) - Y_i(0) | \tau(X_i) = c]$.
    Our estimates in the fourth row of Table \ref{table_ope_mercari} therefore can be interpreted as a measure for the cost effectiveness of the counterfactual policy that slightly lowers the threshold $c$.
    Without the constant conditional effect assumption, the result still suggests that marginally expanding the campaign would be profitable.
% between actions $a$ and $1$.
\begin{comment}
The estimates may still allow us to derive a meaningful policy implication.
For example, let us focus on the binary-action case, i.e., $\mathcal A \coloneqq \{0,1\}$. (Our application in Section 6 falls into this category; $A_i=1$ means giving a coupon to user $i$.)
Although we cannot estimate the conditional average effect $E[Y(1) - Y(0)|X=x]$ for each possible vector of the user characteristics $x\in {\cal X}$, we can still estimate the \emph{local average effect}. Suppose that the logging policy $ML$ is a threshold policy such that $ML(1 \mid x) = 1$ iff $\tau(x) \geq c$ for some score function $\tau$ and threshold $c$. (This is the case in our application. We also assume $\tau(X)$ is continuous.)
Then, we can estimate the effect of coupon distribution for users near the decision boundary, i.e., $\Pi \coloneqq E[Y(1) - Y(0) \mid \tau(X) = c]$.
By giving an extra coupon to a customer whose score is slightly below the threshold $c$, the firm can increase its profit by 
$\Pi$.
\end{comment}

\section{Conclusion}\label{conclusion}

We develop an OPE method for a class of logging policies including deficient support ones.
Our method is based on the newly developed ``Approximate Propensity Score.''
We prove that our estimator is consistent and demonstrate its practical performance through simulations and a real-world application.
Promising directions for future work include developing a data-driven procedure to optimize the bandwidth.
Also, the assumption of constant conditional mean reward differences may not be plausible in some applications.
It will be challenging but interesting to relax this assumption to allow for certain types of heterogeneity.
Finally, we look forward to applications of our method in a variety of business, policy, and scientific domains using machine learning.

%We assess the performance of our method using simulation experiments under different parameters and assumptions. We show that our method performs well even with high dimensional data, across a wide range of $\delta$. The simulations also confirm our theoretical result that the estimator is unbiased under constant conditional mean differences.

\bibliography{reference}

%%%%%%%%%%%%%%%%%%%%%%%%%%%%%%%%%%%%%%%%%%%%%%%%%%%%%%%%%%%%%%%%%%%%%%%%%%%%%%%
%%%%%%%%%%%%%%%%%%%%%%%%%%%%%%%%%%%%%%%%%%%%%%%%%%%%%%%%%%%%%%%%%%%%%%%%%%%%%%%
% APPENDIX
%%%%%%%%%%%%%%%%%%%%%%%%%%%%%%%%%%%%%%%%%%%%%%%%%%%%%%%%%%%%%%%%%%%%%%%%%%%%%%%
%%%%%%%%%%%%%%%%%%%%%%%%%%%%%%%%%%%%%%%%%%%%%%%%%%%%%%%%%%%%%%%%%%%%%%%%%%%%%%%
\newpage
\appendix
\onecolumn

\section{Regularity Conditions and Discussion\label{appendix:discuss_assump}}

\begin{definition}[Twice continuously differentiable]
We say that a bounded open set $S\subset \mathbb{R}^p$ is {\it twice continuously differentiable} if for every $x\in S$, there exists a ball $B(x,\epsilon)$ and a one-to-one mapping $\psi$ from $B(x,\epsilon)$ onto an open set $D\subset \mathbb{R}^{p}$ such that $\psi$ and $\psi^{-1}$ are twice continuously differentiable, $\psi(B(x,\epsilon)\cap S)\subset \{(x_1,...,x_p)\in\mathbb{R}^p:x_p>0\}$ and $\psi(B(x,\epsilon)\cap \partial S)\subset \{(x_1,...,x_p)\in\mathbb{R}^p:x_p=0\}$, where $\partial S$ is the boundary of $S$.
\end{definition}

\begin{definition}[$k$-dimensional Hausdorff measure]
	The $k$-dimensional Hausdorff measure on $\mathbb{R}^{p}$ is defined as follows.
	Let $\Sigma$ be the Lebesgue $\sigma$-algebra on $\mathbb{R}^{p}$ (the set of all Lebesgue measurable sets on $\mathbb{R}^{p}$).
		For $S\in \Sigma$ and $\delta>0$, let
		$
		{\cal H}_\delta^k(S)=\inf\{\sum_{j=1}^\infty d(E_j)^k: S\subset \cup_{j=1}^\infty E_j, d(E_j)<\delta, E_j\subset \mathbb{R}^p$ for all $j$\},
		where $d(E)=\sup\{\|x-y\|: x,y\in E\}$.
		The $k$-dimensional Hausdorff measure of $A$ on $\mathbb{R}^{p}$ is ${\cal H}^k(S)=\lim_{\delta\rightarrow 0}{\cal H}_\delta^k(S)$.
\end{definition}

Our consistency result uses the following assumptions for the subsample assigned to one of the actions $a$ and $1$, for every $a\in\{2,...,m\}$.
Let
\begin{align*}
\mathcal{X}_{a,1}
&\coloneqq
\{x\in{\cal X}: ML(a|x)>0 ~ \text{ or } ~ ML(1|x)>0\}, \\
\widetilde{ML}(a|x)&\coloneqq
\Pr(A_i=a|A_i\in\{1,a\}, X_i=x) =\frac{ML(a|x)}{ML(a|x)+ML(1|x)}
,\\
\mathcal{X}_{a,1}^{a}
&\coloneqq
\{x\in{\cal X}: \widetilde{ML}(a|x)=1\}, \\ \mathcal{X}_{a,1}^{1}
&\coloneqq
\{x\in{\cal X}: \widetilde{ML}(a|x)=0\}.
\end{align*}
In other words, $\mathcal{X}_{a,1}$ is the set of context values for which action $1$ or $a$ can be taken, $\widetilde{ML}(a|x)$ is the probability of choosing action $a$ conditional on $A_i\in\{1,a\}$ and $X_i=x$, and $\mathcal{X}_{a,1}^{a}$ and $\mathcal{X}_{a,1}^{1}$ are the set of context values for which the conditional probability is $1$ and $0$, respectively.

\begin{assumption}[Regularity conditions]\label{consassump_complete}

The following holds for all $a\in\{2,...,m\}$.
	\begin{enumerate}[leftmargin=*,nolistsep,label={(\alph*)}]
	    \item \emph{(Existence of Subsample)}\label{assumption:existence_subsample} $\Pr(A_i\in\{1,a\})>0$.
	    \item{\rm (Almost Everywhere Continuity of $ML$)} \label{assumption:ML-continuity} $ML(a|\cdot)$ and $ML(1|\cdot)$ are continuous almost everywhere on ${\cal X}_{a,1}$ with respect to the Lebesgue measure.
		\item{\rm (Measure Zero Boundaries of ${\cal X}_{a,1}^a$ and ${\cal X}_{a,1}^1$)}. \label{assumption:ML-measure-zero-boundary} For $a'\in\{1,a\}$, ${\cal L}^p({\cal X}_{a,1}^{a'})={\cal L}^p({\rm int}({\cal X}_{a,1}^{a'}))$, where ${\cal L}^p$ is the Lebesgue measure on $\mathbb{R}^p$.
		\item \emph{(Finite Moments)}\label{finite-moments} $E[Y_i^2] < \infty$.
		\item \emph{(Nonzero Conditional Variance)}\label{nonzero-cond-var} If $\Pr(\widetilde{ML}(a|X_i) \in (0,1)| A_i\in\{1,a\}) > 0$, then $\Var(\widetilde{ML}(a|X_i)|\widetilde{ML}(a|X_i) \in (0,1),A_i\in\{1,a\}) > 0$.
	\end{enumerate}
	If $\Pr(\widetilde{ML}(a|X_i)\in (0,1)| A_i\in\{1,a\})=0$, then the following conditions \ref{assumption:deterministic}--\ref{assumption:boundary-continuity} additionally hold.
	\begin{enumerate}[leftmargin=*,nolistsep,label={(\alph*)}]
		\renewcommand{\theenumi}{(\alph{enumi})}
		\renewcommand{\labelenumi}{(\alph{enumi})}
		\setcounter{enumi}{5}
		%\item {\rm (Nonzero Variance)}\label{assumption:nondegeneracy}
		%$\Var (\widetilde{ML}(a|X_i)|A_i\in\{1,a\})>0$.
		\item {\rm (Deterministic $ML$)}\label{assumption:deterministic}
		For all $x\in\mathbb{R}^p$, either $ML(a|x)=1$ or $ML(a|x)=0$.
		%${\cal X}_{a,1}=\{x\in{\cal X}:ML(a|x)=1 ~ \text{ or } ML(1|x)=1\}$.
		%If $\widetilde{ML}(a|x)=1$, then $ML(a|x)=1$.
		%If $\widetilde{ML}(a|x)=0$, then $ML(1|x)=1$.
		\item {\rm ($C^2$ Boundary of $\Omega^*_a$)}\label{assumption:boundary-partition}
		There exists a partition $\{\Omega^*_{a,1},...,\Omega^*_{a,K}\}$ of $\Omega^*_a=\{x\in \mathbb{R}^p: ML(a|x)=1\}$ (the set of the context values for which the probability of choosing action $a$ is one) such that
		\begin{enumerate}[nolistsep,label=(\arabic*)]
			\item ${\rm dist}(\Omega^*_{a,k},\Omega^*_{a,l})>0$ for any $k,l\in\{1,...,K\}$ such that $k\neq l$. Here ${\rm dist}(S,T)=\inf_{x\in S, y\in T}\|x-y\|$ is the distance between two sets $S$ and $T\subset \mathbb{R}^p$;
			\item $\Omega^*_{a,k}$ is nonempty, bounded, open, connected and twice continuously differentiable for each $k\in \{1,...,K\}$.
% 			\footnote{See Appendix~\ref{appendix:definitions} for the definition.}
% 			\footnote{We say that a bounded open set $S\subset \mathbb{R}^p$ is {\it twice continuously differentiable} if for every $x\in S$, there exists a ball $B(x,\epsilon)$ and a one-to-one mapping $\psi$ from $B(x,\epsilon)$ onto an open set $D\subset \mathbb{R}^{p}$ such that $\psi$ and $\psi^{-1}$ are twice continuously differentiable, $\psi(B(x,\epsilon)\cap S)\subset \{(x_1,...,x_p)\in\mathbb{R}^p:x_p>0\}$ and $\psi(B(x,\epsilon)\cap \partial S)\subset \{(x_1,...,x_p)\in\mathbb{R}^p:x_p=0\}$, where $\partial S$ is the boundary of $S$.}
		\end{enumerate}
		\item {\rm (Regularity of Deterministic $ML$)} \label{assumption:boundary-measure}
		\begin{enumerate}[nolistsep,label=(\arabic*)]
		\item \label{assumption:p-1dim} ${\cal H}^{p-1}(\partial\Omega^*_a)<\infty$, $\int_{\partial\Omega^*_a\cap \partial{\cal X}_{a,1}} d{\cal H}^{p-1}(x)=0$, and $\int_{\partial\Omega^*_a\cap {\cal X}_{a,1}} f_X(x) d{\cal H}^{p-1}(x)>0$, where $\partial S$ denotes the boundary of a set $S\subset\mathbb{R}^p$, $f_{X}$ is the probability density function of $X_i$, and ${\cal H}^{k}$ is the $k$-dimensional Hausdorff measure on $\mathbb{R}^{p}$.
% 		\footnote{See Appendix~\ref{appendix:definitions} for the definition.}
% 		\footnote{The $k$-dimensional Hausdorff measure on $\mathbb{R}^{p}$ is defined as follows.
% 	Let $\Sigma$ be the Lebesgue $\sigma$-algebra on $\mathbb{R}^{p}$ (the set of all Lebesgue measurable sets on $\mathbb{R}^{p}$).
% 		For $S\in \Sigma$ and $\delta>0$, let
% 		$
% 		{\cal H}_\delta^k(S)=\inf\{\sum_{j=1}^\infty d(E_j)^k: S\subset \cup_{j=1}^\infty E_j, d(E_j)<\delta, E_j\subset \mathbb{R}^p$ for all $j$\},
% 		where $d(E)=\sup\{\|x-y\|: x,y\in E\}$.
% 		The $k$-dimensional Hausdorff measure of $A$ on $\mathbb{R}^{p}$ is ${\cal H}^k(S)=\lim_{\delta\rightarrow 0}{\cal H}_\delta^k(S)$.}
		\item \label{assumption:zero-set} There exists $\delta>0$ such that $ML(a|x)=1$ or $ML(1|x)=1$ for almost every $x\in N({\mathcal{X}_{a,1}}, \delta)\cap N(\partial\Omega^*_a,\delta)$, where $N(S,\delta)=\{x\in \mathbb{R}^p: \|x-y\|< \delta \text{ for some $y\in S$}\}$ for a set $S\subset\mathbb{R}^p$ and $\delta>0$.
		\end{enumerate}
		\item {\rm (Conditional Moments and Density near $\partial\Omega^*_a$)} \label{assumption:boundary-continuity}
		There exists $\delta>0$ such that
		\begin{enumerate}[nolistsep,label=(\arabic*)]
			\item$E[Y_{i}(a)|X_i]$, $E[Y_{i}(1)|X_i]$, and $f_{X}$ are continuous and bounded on $N(\partial\Omega^*_a,\delta)$;
			\item $E[Y_{i}(a)^2|X_i]$ and $E[Y_{i}(1)^2|X_i]$ are bounded on $N(\partial\Omega^*_a,\delta)$.
			%\item $E[Y_i^{4}|X_i \in \mathcal{X}_{a,1}]$ is bounded on $N(\partial\Omega^*,\delta)$.
		\end{enumerate}
	\end{enumerate}
\end{assumption}

\subsection*{Discussion}

    Assumption \ref{consassump_complete} \ref{assumption:existence_subsample}--\ref{nonzero-cond-var} are a set of conditions we require for proving consistency of $\hat\beta_a$ when $\widetilde{ML}(1|x)>0$ and $\widetilde{ML}(a|x)>0$ for some $x\in{\cal X}$.
    Assumption \ref{consassump_complete} \ref{assumption:ML-continuity} allows the function $ML$ to be discontinuous on a set of points with the Lebesgue measure zero.
    For example, $ML$ is allowed to be a discontinuous step function as long as it is continuous almost everywhere.
    Assumption \ref{consassump_complete} \ref{assumption:ML-measure-zero-boundary} holds if the Lebesgue measures of the boundaries of ${\cal X}_{a,1}^a$ and ${\cal X}_{a,1}^1$ are zero.

	 Assumption \ref{consassump_complete} \ref{nonzero-cond-var} rules out potential multicollinearity. 
	 If the support of $\widetilde{ML}(a|X_i)$ contains only one value in $(0,1)$, $q_{\delta_n}^{ML}(a|X_i)$ is asymptotically constant and equal to $\widetilde{ML}(a|X_i)$ conditional on $q_{\delta_n}^{ML}(a|X_i)\in (0,1)$, resulting in  multicollinearity between
	 $q_{\delta_n}^{ML}(a|X_i)$ and the intercept.
	 Although dropping the intercept from the linear regression (\ref{olstrue}) solves this issue, Assumption \ref{consassump_complete} \ref{nonzero-cond-var} allows us to only consider the regression with an intercept for the purpose of simplyfing the presentation.
	 
	Assumption \ref{consassump_complete} \ref{assumption:deterministic}--\ref{assumption:boundary-continuity} are a set of additional conditions we require for proving consistency of $\hat\beta_a$ when $\widetilde{ML}(a|x)$ is either $0$ or $1$.
	In particular, we assume by Part \ref{assumption:deterministic} that the original logging policy $ML$ is deterministic and the context space is partitioned into $m$ groups based on the action that the logging policy chooses.
	$\partial\Omega_a^*$ then corresponds to the decision boundary for action $a$.
	In this case, the subsample for which $q_{\delta_n}^{ML}(a| X_i)\in (0,1)$ is contained by the $\delta_n$-neighborhood of $\partial\Omega_a^*$.
	
	Assumption \ref{consassump_complete} \ref{assumption:boundary-partition} imposes the differentiability of $\Omega^*_a$.
	The conditions are satisfied if, for example, $\Omega^*_a=\{x\in \mathbb{R}^p:f(x)\ge 0\}$ for some twice continuously differentiable function $f:\mathbb{R}^p\rightarrow \mathbb{R}$ such that the gradient $\nabla f(x)$ is nonzero for all $x\in \mathbb{R}^p$ with $f(x)=0$.
	%$\Omega^*_a$ takes this form, for example, when the conditional treatment effect $E[Y_i(1)-Y_i(0)|X]$ is predicted by supervised learning based on smooth models such as lasso and ridge regressions, and treatment is recommended to individuals who are estimated to experience nonnegative treatment effects.
	In general, the differentiability of $\Omega^*_a$ may not hold.
	For example, if tree-based algorithms are used to partition the context space, the decision boundary $\partial\Omega^*_a$ is not differentiable.
	However, Assumption \ref{consassump_complete} \ref{assumption:boundary-partition} approximately holds in that $\Omega^*_a$ is arbitrarily well approximated by a set that satisfies the differentiability condition.
	
	Part \ref{assumption:p-1dim} of Assumption \ref{consassump_complete} \ref{assumption:boundary-measure} says that $\partial\Omega^*_a$ is $(p-1)$ dimensional and has nonzero density.
	Part \ref{assumption:zero-set} requires that the logging policy choose either action $1$ or $a$ near the boundary of $\Omega^*_a$ even if the context value is not in the subsample ${\cal X}_{a,1}$ as long as it is in the neighborhood of ${\cal X}_{a,1}$.
	%Part \ref{assumption:zero-set} puts a restriction on the values $ML$ takes on outside the support of $X_i$.
	%It requires that $ML(x)=0$ for almost every $x\notin\Omega^*$ that is outside ${\cal X}$ but is in the neighborhood of ${\cal X}$.
	%$ML(x)$ may take on any value if $x$ is not close to ${\cal X}$.
	%These conditions hold in practice.
	
	Lastly, Assumption \ref{consassump_complete} \ref{assumption:boundary-continuity} imposes continuity and boundedness on the conditional moments of rewards and the probability density near the boundary of $\Omega^*_a$.

\section{Simulation Experiments: Details and Additional Results}\label{appendix:sim}
\subsection{Implementation Details}\label{appendix:sim_details}

{\bf Parameter Choice.}
For the variance-covariance matrix $\boldsymbol\Sigma$ of $X_i$, we first create a $100 \times 100$ symmetric matrix $\boldsymbol V$ such that the diagonal elements are one, $\boldsymbol V_{ij}$ is nonzero and equal to $\boldsymbol V_{ji}$ for $(i, j) \in \{2,3,4,5,6\} \times \{35, 66, 78\}$, and everything else is zero. We draw values from $\text{Unif}(-0.5, 0.5)$ independently for the nonzero off-diagonal elements of $\boldsymbol V$. We then create matrix $\boldsymbol \Sigma = \boldsymbol V \times \boldsymbol V$, which is positive semidefinite.

For $\alpha_0$ and $\alpha_{a}$, we first draw $\tilde{\alpha}_{0,j}$, $j = 51, \ldots, 100$ from $\text{Unif}(-100, 100)$ independently across $j$, and draw $\tilde{\alpha}_{a,j}$, $j = 1, \ldots, 100$ from $\text{Unif}(-150,200)$ independently across $j$ and actions $a$. We then set $\tilde{\alpha}_{0,j} = \frac{1}{5} \sum_{a = 1}^{5} \tilde{a}_{a,j}$ for $j = 1,\ldots, 50$ and calculate $\alpha_0$ and $\alpha_{a}$ by normalizing $\tilde{\alpha}_{0}$ and $\tilde{\alpha}_{a}$ such that $\var(\sum_{k=1}^{100}X_{ki}\alpha_{0,k}) = \var(\sum_{k=1}^{100}X_{ki}\alpha_{a,k}) = 1$ for all actions $a$.

{\bf Independent Training Sample $\tilde{\mathcal{D}}$.}
Before simulating 1,000 hypothetical samples, we construct %$\tau_{pred}^{ML}$ using
an independent sample $\tilde{\mathcal{D}}=\{(\tilde Y_i, \tilde X_i, \tilde A_i)\}_{i=1}^{\tilde n}$ of size $\tilde n =$ 10,000. The distribution of $(\tilde Y_i, \tilde X_i, \tilde A_i)$ is the same as that of $(Y_i, X_i, A_i)$ except that (1) $\tilde Y_i(a)$ is generated by $\tilde Y_i(a) =  \sum_{k=1}^{100}X_{ki}^2(0.75\alpha_{0,k}+0.5\alpha_{a,k}) + 0.25 u_i + 0.5\epsilon_{i}(a)$,  where $\epsilon_{i}(a) \sim N(0,1)$ and $\alpha_a=(\alpha_{a,1},...,\alpha_{a,100})\in \mathbb{R}^{100}$, and (2) $\Pr(\tilde A_i = a) = 1/5$ for all actions $a$. This can be viewed as data from a past A/B test conducted to construct a policy.

{\bf Construction of Reward Prediction Functions $\tau_{pred}^{ML}$ and $\tau_{pred}^{\pi}$.}
We use $\tilde{\mathcal{D}}$ to fit a linear model $\tilde Y_i=\sum_{a = 1}^{5} (b_{a}+\sum_{k=1}^{100}\tilde X_{ki}  c_{a,k}) 1\left\{\tilde A_i = a\right\} + e_i$
and compute $\tau_{pred}^{ML}(x, a) = \hat{b}_a + \sum_{k=1}^{100}\tilde x_k  \hat c_{a,k}$.
We repeat this process using a new set of $\tilde n$ independent draws of $\tilde A_i$ to construct $\tau_{pred}^{\pi}$.
%This mimics a situation in which the decision maker conducts two experiments that randomly assigns $A_i$ to estimate the conditional average effect of $A_i$, and then constructs the respective algorithms $ML$ and $\pi$ to greedily choose the treatment predicted to perform better based on the predicted effect.
We construct $\tau_{pred}^{ML}$ and $\tau_{pred}^{\pi}$ only once, and use them for all of the 1,000 samples.

{\bf Training Upper Confident Bound.}
We use $\tilde{\mathcal{D}}$
to train an Upper Confidence Bound bandit algorithm as follows.
Let $D_1(x_1)\in\{1,...,10\}$ indicate which decile of $X_{1i}$ the individual with $\tilde X_{1i}=x_1$ belongs to. Define $D_2(x_2)$ analogously for $\tilde X_{2i}$.
Let $Q(a,d_1,d_2)$ be the sample mean reward for each action $a$ for every decile pair $(d_1,d_2)$ in the distribution of $\tilde X_{1i}$ and $\tilde X_{2i}$.
We then compute
$UCB(x, a) = Q(a, D_1(x_1), D_2(x_2)) + c\sqrt{\frac{\log \tilde n}{\tilde N_{a, D_1(x_1), D_2(x_2)}}}$.
Here, we set exploration parameter $c$ to $2$. $\tilde n(=10,000)$ is the size of the training sample $\tilde{\mathcal{D}}$.
$\tilde N_{a, d_1, d_2}$ is the number of observations with action $a$ for the decile pair $(d_1,d_2)$ in the sample.

\subsection{Additional Results}

We also consider the case in which the conditional mean reward differences are not constant over $x$:
$Y_i(a)$ is generated as $Y_i(a) =  \sum_{k=1}^{100}X_{ki}^2(0.75\alpha_{0,k}+\alpha_{a,k}) + 0.25 u_i$, where $\alpha_a=(\alpha_{a,1},...,\alpha_{a,100})\in \mathbb{R}^{100}$.
The rest of the experiment setup is the same as that in Section \ref{section:sim}.

Table \ref{tbl:app_sim_result} reports the result.
Our method does not necessarily outperform the alternatives, suggesting a limitation of our method when the conditional mean reward differences depend on the context.

\begin{table}[!t]
	\centering
	\caption{Simulation results: non-constant conditional mean reward differences} 
	\label{tbl:app_sim_result}
	\begin{threeparttable}
	\resizebox{\textwidth}{!}{\begin{tabular}{lccccccc}
		\toprule \\ [-2.8ex]
		& \multicolumn{4}{c}{Our Proposed Method with APS Controls} & \multicolumn{2}{c}{Method with Mean Differences} & Direct \\ 
		\cline{2-5} \cline{6-7} \\[-1.8ex]
		 & $\delta = 0.1$ & $\delta=0.5$ & $\delta = 1$  & $\delta = 2.5$ &  A/B Test Sample &  Full Sample &  Method \\
		 & (1) & (2) & (3) & (4) & (5) & (6) & (7) \\
		\hline \\ [-2.0ex]
		\multicolumn{8}{c}{Experiment 1: Mix of A/B Test and Deterministic Policy} \\ \\ [-1.8ex]
		%\multicolumn{5}{l}{Parameter value: $V(\pi) = 2.522$} \\ \\ [-1.0ex]
		\hspace{3mm} Bias & $-.012$ & .015 & .022 & .018 & $-.026$ & $-.006$ &---\\ \\ [-1.8ex]
		 \hspace{3mm} S.D. & .045 & .038 & .033 & .026 & .045 & .018 &---\\\\ [-1.8ex]
		 \hspace{3mm} RMSE & .046 & .041 & .040 & .031 & .052 & .019 &---\\\\ [-1.0ex]
		 \hspace{3mm} Avg. $N$ & 1806 & 6009 & 11627 & 30136 & 500 & 50000 &--- \\ \\ [-1.0ex]
		\multicolumn{8}{c}{Experiment 2: Upper Confidence Bound} \\ \\ [-1.8ex]
		%\multicolumn{5}{l}{Parameter value: $V(\pi) = 3.100$} \\ \\ [-1.0ex]
		\hspace{3mm} Bias & $-.118$ & $-.119$ & $-.121$ & $-.119$ &--- &--- & $-.117$ \\ \\ [-1.8ex]
		 \hspace{3mm} S.D. & .027 & .012 & .009 & .006 &---&---& .006 \\\\ [-1.8ex]
		 \hspace{3mm} RMSE & .121 & .120 & .121 & .119 &---&---& .117  \\\\ [-1.0ex]
		 \hspace{3mm} Avg. $N$ & 3397 & 17343 & 31107 & 47601 &---&---& 50000 \\
		\bottomrule
	\end{tabular}}
	\vspace{0.1em}
	\caption*{
	%\begin{flushleft}
	\footnotesize {\it Notes}: This table shows the bias, the standard deviation (S.D.), and the root mean squared error (RMSE) of the estimators of the reward from the counterfactual policy $V(\pi)$ in the two simulation experiments.
	We use $1,000$ simulations of a size $50,000$ sample to compute these statistics.
	Columns (1)--(4) report estimates from our method with several choices of $\delta$.
	Each APS is computed by averaging 100 simulation draws of the $ML$ value.
	In columns (5)--(6), we estimate the mean reward differences $\beta(a,1)$ by the sample mean differences in the A/B test segment and the full sample, respectively.
	In column (7), we estimate $\beta(a,1)$ by fitting a linear model that predicts the reward from the context and action.
	%These statistics are computed with the estimand set to the average reward $Y$ under the counterfactual algorithm assignment, in either the case where constant conditional mean differences hold or not.
	%We use $\delta\in\{0.1,0.5,1,2.5\}$ to compute APS.
	The bottom row of each panel shows the average number of observations with nonzero APS for every action (Columns (1)--(4)), that with nonzero $ML$ for every action (Column (5)), or the total sample size (Columns (6)--(7)).
%\end{flushleft}
}
\end{threeparttable}
\end{table}

\section{Approximate Propensity Score with Discrete Context Variables}\label{section:discrete-X}

	In this section, we provide the definition of APS when $X_i$ includes discrete context variables.
	Suppose that $X_i=(X_{di},X_{ci})$, where $X_{di}\in \mathbb{R}^{p_d}$ is a vector of discrete context variables, and $X_{ci}\in \mathbb{R}^{p_c}$ is a vector of continuous context variables.
	Let ${\cal X}_d$ denote the support of $X_{di}$ and be assumed to be finite.
	We also assume that $X_{ci}$ is continuously distributed conditional on $X_{di}$.

	We define APS as follows: for each $x=(x_d,x_c)\in{\cal X}$ and $a\in{\cal A}$,
	\begin{align*}
		p^{ML}_\delta(a|x)&\coloneqq\frac{\int_{B(x_c,\delta)}ML(a|x_d,x_c^{*})dx_c^*}{\int_{B(x_c,\delta)}dx_c^*},\\
		p^{ML}(a|x)&\coloneqq \lim_{\delta\rightarrow 0}p^{ML}_\delta(a|x), 
	\end{align*}
	where $B(x_c, \delta)=\{x_c^*\in\mathbb{R}^{p_c}:\|x_c-x_c^*\|\le\delta\}$ is the $\delta$-ball around $x_c\in \mathbb{R}^{p_c}$.
	In other words, we take the average of the $ML(a|x_d,x_c^*)$ values when $x_c^*$ is uniformly distributed on $B(x_c,\delta)$ holding $x_d$ fixed, and let $\delta\rightarrow 0$.

\section{Notations and Lemmas}\label{appendix:notations-lemmas}

\subsection{Basic Notations}
For a vector or matrix $X$, we use $X'$ to denote its transpose.

For a scalar-valued differentiable function $f:A\subset\mathbb{R}^n\rightarrow \mathbb{R}$, let $\nabla f:\mathbb{R}^n\rightarrow \mathbb{R}^n$ be a gradient of $f$: for every $x\in A$,
$$
\nabla f(x)=\left(\frac{\partial f(x)}{\partial x_1}, \cdots, \frac{\partial f(x)}{\partial x_n}\right)'.
$$
Also, when the second-order partial derivatives of $f$ exist, let $D^2 f(x)$ be the Hessian matrix:
$$
D^2 f(x)=\begin{bmatrix}\frac{\partial^2 f(x)}{\partial x_1^2} & \cdots & \frac{\partial^2 f(x)}{\partial x_1\partial x_n}\\
\vdots&\ddots&\vdots\\
\frac{\partial^2 f(x)}{\partial x_n\partial x_1} & \cdots & \frac{\partial^2 f(x)}{\partial x_n^2}
\end{bmatrix}
$$
for each $x\in A$.

Let $f:A\subset\mathbb{R}^m\rightarrow \mathbb{R}^n$ be a function such that its first-order partial derivatives exist.
For each $x\in A$, let $Jf(x)$ be the Jacobian matrix of $f$ at $x$:
$$
Jf(x)=\begin{bmatrix}\frac{\partial f_1(x)}{\partial x_1} & \cdots & \frac{\partial f_1(x)}{\partial x_m}\\
\vdots&\ddots&\vdots\\
\frac{\partial f_n(x)}{\partial x_1} & \cdots & \frac{\partial f_n(x)}{\partial x_m}
\end{bmatrix}.
$$

For a positive integer $n$, let $I_n$ denote the $n\times n$ identity matrix.

\subsection{Differential Geometry}\label{section:differential-geometry}

We provide some concepts and facts from differential geometry of twice continuously differentiable sets, following \cite{Crasta2007SDF}.
Let $A\subset \mathbb{R}^p$ be a twice continuously differentiable set.
%For each $x\in \partial A$, we denote by $T_A(x)\subset \mathbb{R}^p$ the tangent space of $\partial A$ at $x$ and by $\nu_A(x)\in \mathbb{R}^p$ the inward unit normal vector of $\partial A$ at $x$, that is, the unit vector orthogonal to all vectors in $T_A(x)$ that points toward the inside of $A$.
For each $x\in \partial A$, we denote by $\nu_A(x)\in \mathbb{R}^p$ the inward unit normal vector of $\partial A$ at $x$, that is, the unit vector orthogonal to all vectors in the tangent space of $\partial A$ at $x$ that points toward the inside of $A$.
For a set $A\subset\mathbb{R}^p$, let $d_A^s:\mathbb{R}^p\rightarrow \mathbb{R}$ be the signed distance function of $A$, defined by
\begin{align*}
	d_A^s(x) = \begin{cases}
		d(x,\partial A) & \ \ \ \text{if $x\in {\rm cl}(A)$}\\
		-d(x,\partial A) & \ \ \ \text{if $x\in \mathbb{R}^p\setminus {\rm cl}(A)$},
	\end{cases}
\end{align*}
where $d(x,B)=\inf_{y\in B} \|y-x\|$ for any $x\in \mathbb{R}^p$ for a set $B\subset \mathbb{R}^p$.
Note that we can write $N(\partial A,\delta)=\{x\in \mathbb{R}^p:-\delta<d_A^s(x)<\delta\}$ for $\delta>0$.
Lastly, let $\Pi_{\partial A}(x)=\{y\in\partial A:\|y-x\|=d(x,\partial A)\}$ be the set of projections of $x$ on $\partial A$.
\begin{lemma}[Corollary of Theorem 4.16, \cite{Crasta2007SDF}]\label{lemma:Crasta2007SDF-C2}
	Let $A\subset\mathbb{R}^p$ be nonempty, bounded, open, connected and twice continuously differentiable.
	Then the function $d_A^s$ is twice continuously differentiable on $N(\partial A,\mu)$ for some $\mu>0$.
	In addition, for every $x_0\in \partial A$, $\Pi_{\partial A}(x_0+t\nu_A(x_0))=\{x_0\}$ for every $t\in(-\mu,\mu)$.
	Furthermore, for every $x\in N(\partial A,\mu)$, $\Pi_{\partial A}(x)$ is a singleton, $\nabla d_A^s(x)=\nu_A(y)$ and $x=y+d_A^s(x)\nu_A(y)$ for $y\in \Pi_{\partial A}(x)$, and $\|\nabla d_A^s(x)\|=1$.
\end{lemma}

\begin{proof}
	We apply results from \cite{Crasta2007SDF}.
	%Note first that $N(\partial A,\delta)=\{x\in \mathbb{R}^p:-\delta<d_A^s(x)<\delta\}$ for $\delta>0$.
	Let $K=\{x\in\mathbb{R}^p:\|x\|\le 1\}$.
	$K$ is nonempty, compact, convex subset of $\mathbb{R}^p$ with the origin as an interior point.
	The polar body of $K$, defined as $K_0=\{y\in \mathbb{R}^p:y\cdot x\le 1\text{ for all $x\in K$}\}$, is $K$ itself.
	The gauge functions $\rho_K, \rho_{K_0}:\mathbb{R}^p\rightarrow [0, \infty]$ of $K$ and $K_0$ are given by
	\begin{align*}
	\rho_K(x) &\coloneqq {\rm inf}\{t\ge 0:x\in tK\} =\|x\|,\\
	\rho_{K_0}(x) &\coloneqq {\rm inf}\{t\ge 0:x\in tK_0\} =\|x\|.
	\end{align*}
	Given $\rho_{K_0}$, the Minkowski distance from a set $S\subset\mathbb{R}^p$ is defined as
	$$
	\delta_S(x)\coloneqq\inf_{y\in S}\rho_{K_0}(x-y),~~~x\in\mathbb{R}^p.
	$$
	Note that we can write
	\begin{align*}
		d_A^s(x) = \begin{cases}
			\delta_{\partial A}(x)& \ \ \ \text{if $x\in {\rm cl}(A)$}\\
			-\delta_{\partial A}(x)& \ \ \ \text{if $x\in \mathbb{R}^p\setminus {\rm cl}(A)$}.
		\end{cases}
	\end{align*}
	%the signed distance function $d_A^s$ is regarded as the Minkowski distance from the boundary of $A$ in terms of $\rho_{K_0}$.
	It then follows from Theorem 4.16 of \cite{Crasta2007SDF} that $d_A^s$ is twice continuously differentiable on $N(\partial A,\mu)$ for some $\mu>0$, and for every $x_0\in \partial A$,
	\begin{align*}
		\nabla d_A^s(x_0)&=\frac{\nu_A(x_0)}{\rho_K(\nu_A(x_0))}\\
		&=\frac{\nu_A(x_0)}{\|\nu_A(x_0)\|}\\
		&=\nu_A(x_0),
	\end{align*}
	where the last equality follows since $\nu_A(x_0)$ is a unit vector.
	It then follows that $\|\nabla d_A^s(x_0)\|=\|\nu_A(x_0)\|=1$ for every $x_0\in \partial A$.
	Also, it is obvious that, for every $x_0\in \partial A$, $\Pi_{\partial A}(x_0)=\{x_0\}$ and $x_0=x_0+d_A^s(x_0)\nu_A(x_0)$, since $d_A^s(x_0)=0$.
	In addition, as stated in the proof of Theorem 4.16 of \cite{Crasta2007SDF}, $\mu$ is chosen so that (4.7) in Proposition 4.6 of \cite{Crasta2007SDF} holds for every $x_0\in\partial A$ and every $t\in (-\mu,\mu)$.
	That is, $\Pi_{\partial A}(x_0+t\nabla \rho_K(\nu_A(x_0)))=\{x_0\}$ for every $x_0\in\partial A$ and every $t\in (-\mu,\mu)$.
	Since $\nabla \rho_K(\nu_A(x_0))=\frac{\nu_A(x_0)}{\|\nu_A(x_0)\|}=\nu_A(x_0)$, $\Pi_{\partial A}(x_0+t\nu_A(x_0))=\{x_0\}$ for every $x_0\in\partial A$ and every $t\in (-\mu,\mu)$.
	
	Furthermore, for every $x\in N(\partial A,\mu)\setminus \partial A$, $\Pi_{\partial A}(x)$ is a singleton as shown in the proof of Theorem 4.16 of \cite{Crasta2007SDF}.
	Let $\pi_{\partial A}(x)$ be the unique element in $\Pi_{\partial A}(x)$.
	By Lemma 4.3 of \cite{Crasta2007SDF}, for every $x\in N(\partial A,\mu)\setminus \partial A$,
	\begin{align*}
		\nabla d_A^s(x)&=\frac{\nu_A(\pi_{\partial A}(x))}{\rho_K(\nu_A(\pi_{\partial A}(x)))}\\
		&=\frac{\nu_A(\pi_{\partial A}(x))}{\|\nu_A(\pi_{\partial A}(x))\|}\\
		&=\nu_A(\pi_{\partial A}(x)),
	\end{align*}
	where the last equality follows since $\nu_A(\pi_{\partial A}(x))$ is a unit vector.
	It then follows that $\|\nabla d_A^s(x)\|=\|\nu_A(\pi_{\partial A}(x))\|=1$ for every $x\in N(\partial A,\mu)\setminus \partial A$.
	
	Lastly, note that
	\begin{align*}
		\delta_{\partial A}(x) = \begin{cases}
			d_A^s(x) & \ \ \ \text{if $x\in N(\partial A,\mu)\cap {\rm int}(A)$}\\
			-d_A^s(x) & \ \ \ \text{if $x\in N(\partial A,\mu)\setminus {\rm cl}(A)$},
		\end{cases}
	\end{align*}
	and
	\begin{align*}
		\nabla \delta_{\partial A}(x) = \begin{cases}
			\nabla d_A^s(x) & \ \ \ \text{if $x\in N(\partial A,\mu)\cap {\rm int}(A)$}\\
			-\nabla d_A^s(x) & \ \ \ \text{if $x\in N(\partial A,\mu)\setminus {\rm cl}(A)$},
		\end{cases}
	\end{align*}
	so $\delta_{\partial A}(x)\nabla \delta_{\partial A}(x)= d_A^s(x)\nabla d_A^s(x)=d_A^s(x)\nu_A(\pi_{\partial A}(x))$ for every $x\in N(\partial A,\mu)\setminus \partial A$.
	By Proposition 3.3 (i) of \cite{Crasta2007SDF}, for every $x\in N(\partial A,\mu)\setminus \partial A$,
	\begin{align*}
	\nabla \rho_K(\nabla \delta_{\partial A}(x))&=\frac{x-\pi_{\partial A}(x)}{\delta_{\partial A}(x)},
	\end{align*}
	which implies that
	\begin{align*}
		x&=\pi_{\partial A}(x)+\delta_{\partial A}(x)\nabla \rho_K(\nabla \delta_{\partial A}(x))\\
		&=\pi_{\partial A}(x)+\delta_{\partial A}(x)\frac{\nabla \delta_{\partial A}(x)}{\|\nabla \delta_{\partial A}(x)\|}\\
		&=\pi_{\partial A}(x)+d_A^s(x)\nu_A(\pi_{\partial A}(x)).
	\end{align*}
	\end{proof}

We say that a set $A\subset \mathbb{R}^n$ is a \textit{$m$-dimensional $C^1$ submanifold of $\mathbb{R}^{n}$} if for every point $x\in A$, there exist an open neighborhood $V\subset \mathbb{R}^n$ of $x$ and a one-to-one continuously differentiable function $\phi$ from an open set $U\subset \mathbb{R}^{m}$ to $\mathbb{R}^n$ such that the Jacobian matrix $J\phi(u)$ is of rank $m$ for all $u\in U$, and $\phi(U)=V\cap A$.
\begin{lemma}\label{lemma:C1-submanifold}
	Let $A\subset\mathbb{R}^p$ be nonempty, bounded, open, connected and twice continuously differentiable.
	Then $\partial A$ is a $(p-1)$-dimensional $C^1$ submanifold of $\mathbb{R}^{p}$,
\end{lemma}

\begin{proof}
	Fix any $x^*\in \partial A$.
	By Lemma \ref{lemma:Crasta2007SDF-C2}, $\nabla d_{A}^s(x^*)$ is nonzero.
	Without loss of generality, let $\frac{\partial d_ {A}^s(x^*)}{\partial x_p}\neq 0$.
	Let $\psi: \mathbb{R}^p\rightarrow \mathbb{R}^p$ be the function such that $\psi(x)=(x_1,...,x_{p-1},d_ {A}^s(x))$.
	$\psi$ is continuously differentiable, and 
	the Jacobian matrix of $\psi$ at $x^*$ is given by
	\begin{align*}
		J\psi(x^*)=
		\begin{pmatrix}
			\frac{\partial \psi_1}{\partial x_1}(x^*) & \cdots& \frac{\partial \psi_1}{\partial x_p}(x^*)\\
			\vdots &\ddots &\vdots\\
			\frac{\partial \psi_p}{\partial x_1}(x^*) & \cdots& \frac{\partial \psi_p}{\partial x_p}(x^*)
		\end{pmatrix}
		=
		\begin{pmatrix}
			&  & & 0\\
			& I_{p-1} &  &\vdots\\
			&  & &0\\
			\frac{\partial d_ {A}^s(x^*)}{\partial x_1} & \cdots &\frac{\partial d_ {A}^s(x^*)}{\partial x_{p-1}}& \frac{\partial d_ {A}^s(x^*)}{\partial x_p}
		\end{pmatrix}.
	\end{align*}
	%where $I_{p-1}$ is the $(p-1)\times (p-1)$ identity matrix.
	Since $\frac{\partial d_ {A}^s(x^*)}{\partial x_p}\neq 0$, the Jacobian matrix is invertible.
	By the Inverse Function Theorem, there exist an open set $V$ containing $x^*$ and an open set $W$ containing $\psi(x^*)$ such that $\psi: V\rightarrow W$ has an inverse function $\psi^{-1}: W\rightarrow V$ that is continuously differentiable.
	We make $V$ small enough so that $\frac{\partial d_ {A}^s(x)}{\partial x_p}\neq 0$ for every $x\in V$.
	The Jacobian matrix of $\psi^{-1}$ is given by $J\psi^{-1}(y)=J\psi(\psi^{-1}(y))^{-1}$ for all $y\in W$.
	
	Now note that $\psi(x)=(x_1,...,x_{p-1},0)$ for all $x\in V\cap \partial  A$ by the definition of $d_ {A}^s$.
	Let $U=\{(x_1,...,x_{p-1})\in \mathbb{R}^{p-1}: x\in V\cap \partial  A\}$ and $\phi:U\rightarrow \mathbb{R}^p$ be a function such that $\phi(u)=\psi^{-1}((u,0))$ for all $u\in U$.
	Below we verify that $\phi$ is one-to-one and continously differentiable, that $J\phi(u)$ is of rank $p-1$ for all $u\in U$, that $\phi(U)=V\cap \partial  A$, and that $U$ is open.
	
	First, $\phi$ is one-to-one, since $\psi^{-1}$ is one-to-one, and $(u,0)\neq (u',0)$ if $u\neq u'$.
	Second, $\phi$ is continuously differentiable, since $\psi^{-1}$ is so.
	The Jacobian matrix of $\phi$ at $u\in U$ is by definition
	\begin{align*}
		J\phi(u)=
		\begin{pmatrix}
			\frac{\partial \psi_1^{-1}}{\partial y_1}((u,0)) & \cdots& \frac{\partial \psi_1^{-1}}{\partial y_{p-1}}((u,0))\\
			\vdots &\ddots &\vdots\\
			\frac{\partial \psi_p^`-1}{\partial y_1}((u,0)) & \cdots& \frac{\partial \psi_p^{-1}}{\partial y_{p-1}}((u,0))
		\end{pmatrix}.
	\end{align*}
	Note that this is the left $p\times (p-1)$ submatrix of $J\psi^{-1}((u,0))$.
	Since $J\psi^{-1}((u,0))$ has full rank, $J\phi(u)$ is of rank $p-1$.
	Moreover,
	\begin{align*}
		\phi(U)&=\{\psi^{-1}((u,0)):u\in U\}\\
		&=\{\psi^{-1}((x_1,...,x_{p-1},0)): x\in V\cap \partial  A\}\\
		&=\{\psi^{-1}(\psi(x)): x\in V\cap \partial  A\}\\
		&=V\cap \partial  A.
	\end{align*}
	
	Lastly, we show that $U$ is open.
	Pick any $\bar u\in U$.
	Then, there exists $\bar x_p\in\mathbb{R}$ such that $(\bar u,\bar x_p)\in V\cap \partial  A$.
	As $(\bar u,\bar x_p)\in V\cap \partial  A$, $d_ {A}^s((\bar u,\bar x_p))=0$.
	Since $\frac{\partial d_ {A}^s((\bar u,\bar x_p))}{\partial x_p}\neq 0$, it follows by the Implicit Function Theorem that there exist an open set $S\subset \mathbb{R}^{p-1}$ containing $\bar u$ and a continuously differentiable function $g: S\rightarrow \mathbb{R}$ such that $g(\bar u)=\bar x_p$ and $d_ {A}^s(u,g(u))=0$ for all $u\in S$.
	Since $g$ is continuous, $(\bar u, g(\bar u))\in V$ and $V$ is open, there exists an open set $S'\subset S$ containing $\bar u$ such that $(u,g(u))\in V$ for all $u\in S'$.
	By the definition of $d_ {A}^s$, $d_ {A}^s(x)=0$ if and only if $x\in \partial  A$.
	Therefore, if $u\in S'$, $(u,g(u))$ must be contained by $\partial  A$, for otherwise $d_ {A}^s(u,g(u))\neq 0$, which is a contradiction.
	Thus, $(u,g(u))\in V\cap \partial  A$ and hence $u\in U$ for all $u\in S'$.
	This implies that $S'$ is an open subset of $U$ containing $\bar u$, which proves that $U$ is open.
\end{proof}

\subsection{Geometric Measure Theory}\label{section:geometric-theory}

We provide some concepts and facts from geometric measure theory, following \cite{Krantz2008book}.
Recall that for a function $f:A\subset \mathbb{R}^m\rightarrow \mathbb{R}^n$ and a point $x\in A$ at which $f$ is differentiable, $Jf(x)$ denotes the Jacobian matrix of $f$ at $x$.

\begin{lemma}[Coarea Formula, Lemma 5.1.4 and Corollary 5.2.6 of \cite{Krantz2008book}]\label{lemma:coarea}
	If $f:\mathbb{R}^m\rightarrow \mathbb{R}^n$ is a Lipschitz function and $m\ge n$, then
	$$
	\int_Ag(x)J_nf(x)d{\cal L}^m(x) = \int_{\mathbb{R}^n}\int_{\{x'\in A:f(x')=y\}}g(x)d{\cal H}^{m-n}(x)d{\cal L}^n(y)
	$$
	for every Lebesgue measurable subset $A$ of $\mathbb{R}^m$ and every ${\cal L}^m$-measurable function $g:A\rightarrow \mathbb{R}$, where for each $x\in \mathbb{R}^m$ at which $f$ is differentiable,
	$$
	J_n f(x) = \sqrt{{\rm det}((Jf(x))(Jf(x))')}.
	$$
\end{lemma}

Let $A$ be an $m$-dimensional $C^1$ submanifold of $\mathbb{R}^n$.
Let $x\in A$ and let $\phi: U\subset \mathbb{R}^m\rightarrow \mathbb{R}^n$ be as in the definition of $m$-dimensional $C^1$ submanifold.
We denote by $T_A(x)$ the tangent space of $A$ at $x$, $\{J\phi(u) v: v\in \mathbb{R}^m\}$, where $u=\phi^{-1}(x)$.

\begin{lemma}[Area Formula, Lemma 5.3.5 and Theorem 5.3.7 of \cite{Krantz2008book}]\label{lemma:area}
	Suppose $m\le \nu$ and $f:\mathbb{R}^n\rightarrow \mathbb{R}^\nu$ is Lipschitz.
	If $A$ is an $m$-dimensional $C^1$ submanifold of $\mathbb{R}^n$, then
	$$
	\int_Ag(x)J_m^Af(x)d{\cal H}^m(x) = \int_{\mathbb{R}^\nu}\sum_{x\in A:f(x)=y}g(x)d{\cal H}^m(y)
	$$
	for every ${\cal H}^m$-measurable function $g:A\rightarrow \mathbb{R}$, where for each $x\in \mathbb{R}^n$ at which $f$ is differentiable,
	\begin{align*}
		J_m^Af(x)=\frac{{\cal H}^m(\{Jf(x)y:y\in P\})}{{\cal H}^m(P)}
	\end{align*}
	for an arbitrary $m$-dimensional parallelepiped $P$ contained in $T_A(x)$.
\end{lemma}

Let $A\subset\mathbb{R}^p$.
For each $x\in\mathbb{R}^p$ at which $d_{A}^s$ is differentiable and for each $\lambda\in\mathbb{R}$, let $\psi_A(x,\lambda)=x+\lambda\nabla d_{A}^s(x)$.

\begin{lemma}\label{lemma:neighborhood-integral}
	Let $\Omega\subset \mathbb{R}^p$, and suppose that there exists a partition $\{\Omega_{1},...,\Omega_{M}\}$ of $\Omega$ such that
	\begin{enumerate}[label=(\roman*)]
		\item ${\rm dist}(\Omega_m,\Omega_{m'})>0$ for any $m,m'\in\{1,...,M\}$ such that $m\neq m'$;
		\item $\Omega_m$ is nonempty, bounded, open, connected and twice continuously differentiable for each $m\in \{1,...,M\}$.
	\end{enumerate}
	Then there exists $\mu>0$ such that $d_{\Omega}^s$ is twice continuously differentiable on $N(\partial\Omega,\mu)$ and that
	$$
	\int_{N(\partial \Omega,\delta)}g(x)dx = \int_{-\delta}^{\delta}\int_{\partial \Omega}g(u+\lambda \nu_\Omega(u))J_{p-1}^{\partial\Omega}\psi_\Omega(u,\lambda)d{\cal H}^{p-1}(u)d\lambda
	$$
	for every $\delta\in (0,\mu)$ and every function $g:\mathbb{R}^p\rightarrow \mathbb{R}$ that is integrable on $N(\partial \Omega,\delta)$, where for each fixed $\lambda\in(-\mu,\mu)$, $J_{p-1}^{\partial\Omega}\psi_\Omega(\cdot,\lambda)$ is calculated by applying the operation $J_{p-1}^{\partial \Omega}$ to the function $\psi_\Omega(\cdot,\lambda)$.
	Futhermore, $J_{p-1}^{\partial\Omega}\psi_\Omega(x,\cdot)$ is continuously differentiable in $\lambda$ and $J_{p-1}^{\partial\Omega}\psi_\Omega(x,0)=1$ for every $x\in\partial\Omega$, and  $J_{p-1}^{\partial\Omega}\psi_\Omega(\cdot,\cdot)$ and $\tfrac{\partial J_{p-1}^{\partial\Omega}\psi_\Omega(\cdot,\cdot)}{\partial\lambda}$ are bounded on $\partial\Omega\times (-\mu,\mu)$.
\end{lemma}

\begin{proof}
	Let $\bar\mu=\frac{1}{2}\min_{m,m'\in\{1,...,M\},m\neq m'}{\rm dist}(\Omega^*_m,\Omega_{m'})$ so that $\{N(\partial\Omega_m,\bar\mu)\}_{m=1}^M$ is a partition of $N(\partial\Omega,\bar\mu)$.
	Note that for every $m\in\{1,...,M\}$, $d_{\Omega}^s(x)=d_{\Omega_m}^s(x)$ for every $x\in N(\partial\Omega_m,\bar\mu)$.
	By Lemma \ref{lemma:Crasta2007SDF-C2}, for every $m\in\{1,...,M\}$, there exists $\bar\mu_m>0$ such that $d_{\Omega_m}^s$ is twice continuously differentiable on $N(\partial\Omega_m,\bar\mu_m)$.
	Letting $\mu\in(0,\min\{\bar\mu,\bar\mu_1,...,\bar\mu_M\})$, we have that $d_{\Omega}^s$ is twice continuously differentiable on $N(\partial\Omega,\mu)$.
	This implies that $d_{\Omega}^s$ is Lipschitz on $N(\partial\Omega,\mu)$.
	For every $\delta\in (0,\mu)$ and every function $g:\mathbb{R}^p\rightarrow \mathbb{R}$ that is integrable on $N(\partial \Omega,\delta)$,
	\begin{align}
		\int_{N(\partial \Omega,\delta)}g(x)dx &= \int_{\{x'\in\mathbb{R}^p:d_{\Omega}^s(x')\in (-\delta,\delta)\}}g(x)\sqrt{{\rm det}(\|\nabla d_{\Omega}^s(x)\|)}dx \nonumber\\
		&= \int_{\{x'\in\mathbb{R}^p:d_{\Omega}^s(x')\in (-\delta,\delta)\}}g(x)\sqrt{{\rm det}(\nabla d_{\Omega}^s(x)'\nabla d_{\Omega}^s(x))}dx \nonumber\\
		&= \int_{\{x'\in\mathbb{R}^p:d_{\Omega}^s(x')\in (-\delta,\delta)\}}g(x)\sqrt{{\rm det}((J d_{\Omega}^s(x))(J d_{\Omega}^s(x))')}dx \nonumber\\
		&= \int_{\mathbb{R}}\int_{\{x'\in \mathbb{R}^p:d_{\Omega}^s(x')\in (-\delta,\delta), d_{\Omega}^s(x')=\lambda\}}g(x)d{\cal H}^{p-1}(x)d\lambda \nonumber\\
		&= \int_{-\delta}^{\delta}\int_{\{x'\in \mathbb{R}^p: d_{\Omega}^s(x')=\lambda\}}g(x)d{\cal H}^{p-1}(x)d\lambda, \label{eq:coarea}
	\end{align}
	where the first equality follows since $\|\nabla d_{\Omega}^s(x)\|=1$ for every $x\in N(\partial \Omega,\delta)$ by Lemma \ref{lemma:Crasta2007SDF-C2}, the third equality follows from the definition of the Jacobian matrix, and the fourth equality follows from Lemma \ref{lemma:coarea}.
	
	Let $\Gamma(\lambda)=\{x\in \mathbb{R}^p: d_{\Omega}^s(x)=\lambda\}$ for each $\lambda\in (-\mu,\mu)$.
	Since $\nabla d_{\Omega}^s$ is differentiable on $N(\partial\Omega,\mu)$, $\psi_\Omega(x,\lambda)$ is defined on $N(\partial\Omega,\mu)\times \mathbb{R}$.
	We show that $\{\psi_\Omega(x_0,\lambda):x_0\in\partial\Omega\}\subset\Gamma(\lambda)$ for every $\lambda\in (-\mu,\mu)$.
	By Lemma \ref{lemma:Crasta2007SDF-C2}, for every $x_0\in\partial\Omega$, $\psi_\Omega(x_0,\lambda)=x_0+\lambda \nu_\Omega(x_0)$ and
	\begin{align*}
		\Pi_{\partial \Omega}(\psi_\Omega(x_0,\lambda))
		&=\Pi_{\partial \Omega}(x_0+\lambda \nu_\Omega(x_0))\\
		&=\{x_0\}.
	\end{align*}
	Hence,
	\begin{align*}
		d(\psi_\Omega(x_0,\lambda), \partial \Omega)&=\|\psi_\Omega(x_0,\lambda)-x_0\|\\
		&=\|\lambda \nu_\Omega(x_0)\|\\
		&=|\lambda|.
	\end{align*}
	Since $\nu_\Omega(x_0)$ is an inward normal vector, $\psi_\Omega(x_0,\lambda)\in {\rm cl}(A)$ if $0\le \lambda<\mu$, and $\psi_\Omega(x,\lambda_0)\in \mathbb{R}^p\setminus {\rm cl}(A)$ if $-\mu<\lambda< 0$.
	It follows that
	\begin{align*}
		d_A^s(\psi_\Omega(x_0,\lambda)) &= \begin{cases}
			|\lambda|& \ \ \ \text{if $0\le \lambda <\mu$}\\
			-|\lambda| & \ \ \ \text{if $\mu<\lambda<0$}
		\end{cases}\\
		&=\lambda,
	\end{align*}
	so $\{\psi_\Omega(x_0,\lambda):x_0\in\partial\Omega\}\subset\Gamma(\lambda)$.
	It also holds that
	$\Gamma(\lambda)\subset \{\psi_\Omega(x_0,\lambda):x_0\in\partial\Omega\}$, since by Lemma \ref{lemma:Crasta2007SDF-C2}, for every $x\in \Gamma(\lambda)$,
	\begin{align*}
		\psi_\Omega(\pi_{\partial \Omega}(x),\lambda)&=\pi_{\partial \Omega}(x)+\lambda\nabla d_{\Omega}^s(\pi_{\partial \Omega}(x))\\
		&=\pi_{\partial \Omega}(x)+d_{\Omega}^s(x) \nu_\Omega(\pi_{\partial \Omega}(x))\\
		&=x,
	\end{align*}
	where $\pi_{\partial \Omega}(x)$ is the unique element in $\Pi_{\partial \Omega}(x)$.
	Thus, $\{\psi_\Omega(x_0,\lambda):x_0\in\partial\Omega\}=\Gamma(\lambda)$.
	
	Now note that $\{\partial\Omega_m\}_{m=1}^M$ is a partition of $\partial\Omega$, since ${\rm dist}(\Omega_m,\Omega_{m'})>0$ for any $m,m'\in\{1,...,M\}$ such that $m\neq m'$.
	By Lemma \ref{lemma:C1-submanifold}, $\partial\Omega_m$ is a $(p-1)$-dimensional $C^1$ submanifold of $\mathbb{R}^p$ for every $m\in\{1,...,M\}$, and hence $\partial\Omega$ is a $(p-1)$-dimensional $C^1$ submanifold of $\mathbb{R}^p$.
	Furthermore, since $\nabla d_{\Omega}^s$ is continuously differentiable on $N(\partial\Omega,\mu)$, $\psi_\Omega(\cdot,\lambda)$ is continuously differentiable on $N(\partial\Omega,\mu)$, which implies that $\psi_\Omega(\cdot,\lambda)$ is Lipschitz on $N(\partial\Omega,\mu)$ for every $\lambda\in\mathbb{R}$.
	Applying Lemma \ref{lemma:area}, we have that for every $\lambda\in (-\mu,\mu)$,
	\begin{align}
		\int_{\partial \Omega}g(u+\lambda \nu_\Omega(u))J_{p-1}^{\partial\Omega}\psi_\Omega(u,\lambda)d{\cal H}^{p-1}(u)&=
		\int_{\partial\Omega}g(\psi_\Omega(u,\lambda))J_{p-1}^{\partial\Omega}\psi_\Omega(u,\lambda)d{\cal H}^{p-1}(u) \nonumber\\
		& =\int_{\mathbb{R}^p}\sum_{u\in\partial \Omega:\psi_\Omega(u,\lambda)=x}g(\psi_\Omega(u,\lambda))d{\cal H}^{p-1}(x).\label{eq:area}
	\end{align}
	If $x\notin \{\psi_\Omega(u,\lambda):u\in\partial\Omega\}$, $\{u\in\partial \Omega:\psi_\Omega(u,\lambda)=x\}=\emptyset$.
	If $x\in \{\psi_\Omega(u,\lambda):u\in\partial\Omega\}$, there exists $u\in\partial\Omega$ such that $x=\psi_\Omega(u,\lambda)$.
	Since $\Pi_{\partial\Omega}(x)=\Pi_{\partial\Omega}(u+\lambda \nabla d_\Omega^s(u))=\Pi_{\partial\Omega}(u+\lambda \nu_\Omega(u))=\{u\}$ by Lemma \ref{lemma:Crasta2007SDF-C2}, such $u$ is unique, and hence $\{u\in\partial \Omega:\psi_\Omega(u,\lambda)=x\}$ is a singleton.
	 %$x=\psi_\Omega(u,\lambda)=u+\lambda \nabla d_\Omega^s(u)$ for some $u\in\partial\Omega$.
	%Since $\|u-x\|=|\lambda|=d(\psi_\Omega(u,\lambda),\partial\Omega)=d(x,\partial\Omega)$, $u\in\Pi(x)$.
	%By Lemma \ref{lemma:Crasta2007SDF-C2}, $\Pi(x)$ is a singleton, and hence $\{u\in\partial \Omega:\psi_\Omega(u,\lambda)=x\}$ is a singleton.
	It follow that
	\begin{align}
		\int_{\mathbb{R}^p}\sum_{u\in\partial \Omega:\psi_\Omega(u,\lambda)=x}g(\psi_\Omega(u,\lambda))d{\cal H}^{p-1}(x)
		&= \int_{\{\psi_\Omega(u,\lambda):u\in\partial\Omega\}}g(x)d{\cal H}^{p-1}(x) \nonumber\\
		&= \int_{\Gamma(\lambda)} g(x)d{\cal H}^{p-1}(x), \label{eq:integral-level-set}
	\end{align}
	where the last equality holds since $\{\psi_\Omega(u,\lambda):u\in\partial\Omega\}=\Gamma(\lambda)$.
	Combining (\ref{eq:coarea}), (\ref{eq:area}) and (\ref{eq:integral-level-set}), we obtain
	$$
	\int_{N(\partial \Omega,\delta)}g(x)dx = \int_{-\delta}^{\delta}\int_{\partial \Omega}g(u+\lambda \nu_\Omega(u))J_{p-1}^{\partial\Omega}\psi_\Omega(u,\lambda)d{\cal H}^{p-1}(u)d\lambda.
	$$
	
	We next show that $J_{p-1}^{\partial\Omega}\psi_\Omega(x,\cdot)$ is continuously differentiable in $\lambda$ and $J_{p-1}^{\partial\Omega}\psi_\Omega(x,0)=1$ for every $x\in\partial\Omega$.
	Fix an $x\in\partial\Omega$, and let $V_\Omega(x)$ be an arbitrary $p\times (p-1)$ matrix whose columns $v_1(x),...,v_{p-1}(x)\in\mathbb{R}^p$ form an orthonormal basis of $T_{\partial\Omega}(x)$.
	Let $P(x)\subset T_{\partial\Omega}(x)$ be a parallelepiped determined by $v_1(x),...,v_{p-1}(x)$, that is, let $P(x)=\{\sum_{k=1}^{p-1}c_k v_k(x): 0\le c_k\le 1\text{ for $k=1,...,p-1$}\}$.
	Since $v_1(x),...,v_{p-1}(x)$ are linearly independent, $P(x)$ is a $(p-1)$-dimensional parallelepiped.
	It follows that for each fixed $\lambda\in \mathbb{R}$,
	\begin{align*}
		\{J\psi_\Omega(x,\lambda) y:y\in P(x)\}&=\{J\psi_\Omega(x,\lambda) \sum_{k=1}^{p-1}c_k v_k(x): 0\le c_k\le 1\text{ for $k=1,...,p-1$}\}\\
		&=\{\sum_{k=1}^{p-1}c_k J\psi_\Omega(x,\lambda) v_k(x): 0\le c_k\le 1\text{ for $k=1,...,p-1$}\}\\
		&=\{\sum_{k=1}^{p-1}c_k w_k(x,\lambda): 0\le c_k\le 1\text{ for $k=1,...,p-1$}\},
	\end{align*}
	where $w_k(x,\lambda)=J\psi_\Omega(x,\lambda)v_k(x)$ for $k=1,...,p-1$.
	Since $J\psi_\Omega(x,\lambda)v_k(x)$ is the $k$-th column of $J\psi_\Omega(x,\lambda)V_\Omega(x)$, $\{J\psi_\Omega(x,\lambda) y:y\in P(x)\}$ is the parallelepiped determined by the columns of $J\psi_\Omega(x,\lambda)V_\Omega(x)$.
	By Proposition 5.1.2 of \cite{Krantz2008book}, we have that
	\begin{align*}
		J_{p-1}^{\partial\Omega}\psi_\Omega(x,\lambda)&=\frac{{\cal H}^{p-1}(\{\sum_{k=1}^{p-1}c_k w_k(x,\lambda): 0\le c_k\le 1\text{ for $k=1,...,p-1$}\})}{{\cal H}^{p-1}(P(x))}\\
		&=\frac{\sqrt{{\rm det}((J\psi_\Omega(x,\lambda)V_\Omega(x))'(J\psi_\Omega(x,\lambda) V_\Omega(x)))}}{\sqrt{{\rm det}(V_\Omega(x)'V_\Omega(x))}}\\
		&=\frac{\sqrt{{\rm det}((V_\Omega(x)+\lambda D^2d_\Omega^s(x)V_\Omega(x))'(V_\Omega(x)+\lambda D^2d_\Omega^s(x)V_\Omega(x)))}}{\sqrt{{\rm det}(I_{p-1})}}\\
		&=\sqrt{{\rm det}(V_\Omega(x)'V_\Omega(x)+2V_\Omega(x)' \lambda D^2d_\Omega^s(x)V_\Omega(x)+V_\Omega(x)'(\lambda D^2d_\Omega^s(x))^2V_\Omega(x))}\\
		&=\sqrt{{\rm det}(I_{p-1}+\lambda V_\Omega(x)' (2 D^2d_\Omega^s(x)+\lambda (D^2d_\Omega^s(x))^2)V_\Omega(x)))}\\
		&=\sqrt{{\rm det}(I_p+\lambda V_\Omega(x)V_\Omega(x)' (2 D^2d_\Omega^s(x)+\lambda (D^2d_\Omega^s(x))^2))},
	\end{align*}
	where we use the fact that $V_\Omega(x)'V_\Omega(x)=I_{p-1}$ and the fact that ${\rm det}(I_m+AB)={\rm det}(I_n+BA)$ for an $m\times n$ matrix $A$ and an $n\times m$ matrix $B$ (the Weinstein-Aronszajn identity).
	For every $x\in\partial\Omega$, $J_{p-1}^{\partial\Omega}\psi_\Omega(x,\cdot)$ is continuously differentiable in $\lambda$, and
	$J_{p-1}^{\partial\Omega}\psi_\Omega(x,0)=\sqrt{{\rm det}(I_p)}=1$.
	
	Lastly, we show that $J_{p-1}^{\partial\Omega}\psi_\Omega(\cdot,\cdot)$ and $\tfrac{\partial J_{p-1}^{\partial\Omega}\psi_\Omega(\cdot,\cdot)}{\partial\lambda}$ are bounded on $\partial\Omega\times (-\mu,\mu)$.
	Let $f,h:\partial\Omega \times \mathbb{R}^{p \times (p-1)}\rightarrow \mathbb{R}^{p\times p}$ be functions such that
	\begin{align*}
	f(x,A)&=2AA'D^2d_\Omega^s(x),\\
	h(x,A)&=AA'(D^2d_\Omega^s(x))^2.
	\end{align*}
	%where $A$ is a $p\times(p-1)$ matrix with columns $a_1,...,a_{p-1}\in\mathbb{R}^p$.
	Also, let $k:\partial\Omega\times \mathbb{R}\times \mathbb{R}^{p \times (p-1)}\rightarrow \mathbb{R}$ be a function such that
	$$
	k(x,\lambda,A)=\sqrt{{\rm det}(I_p+\lambda f(x,A)+\lambda^2h(x,A))}.
	$$
	Observe that
	\begin{align*}
	J_{p-1}^{\partial\Omega}\psi_\Omega(x,\lambda)=k(x,\lambda,V_\Omega(x))
	\end{align*}
	and that
	\begin{align*}
	&~\frac{\partial J_{p-1}^{\partial\Omega}\psi_\Omega(x,\lambda)}{\partial\lambda}\\
	=&~\left.\frac{\partial k(x,\lambda,A)}{\partial\lambda}\right|_{A=V_\Omega(x)}\\
	=&~\left.\frac{1}{2k(x,\lambda,A)}\sum_{i,j}\frac{\partial {\rm det}(I_p+\lambda f(x,A)+\lambda^2h(x,A))}{\partial b_{ij}}(f_{ij}(x,A)+2\lambda h_{ij}(x,A))\right|_{A=V_\Omega(x)},
	\end{align*}
	where $\tfrac{\partial {\rm det}(B)}{\partial b_{ij}}$ denotes the partial derivative of the function ${\rm det}:\mathbb{R}^{p\times p}\rightarrow \mathbb{R}$ with respect to the $(i,j)$ entry of $B$.
	
	Note that $k(\cdot,\cdot,\cdot)$ and $\tfrac{\partial k(\cdot,\cdot,\cdot)}{\partial\lambda}$ are continuous on $\partial\Omega\times \mathbb{R}\times\mathbb{R}^{p \times (p-1)}$ (except at the points for which $k(x,\lambda,A)=0$), since ${\rm det}$ is infinitely differentiable, and $f$ and $h$ are continuous on $\partial\Omega\times \mathbb{R}^{p \times (p-1)}$.
	Let $S=\{(x,\lambda,A)\in\partial\Omega\times [-\mu,\mu]\times \mathbb{R}^{p \times (p-1)}:\|a_j\|=1 \text{ for $k=1,...,p-1$}\}$, where $a_j$ denotes the $j$th column of $A$.
	Since $k(\cdot,\cdot,\cdot)$ and $\tfrac{\partial k(\cdot,\cdot,\cdot)}{\partial\lambda}$ are continuous and $S$ is closed and bounded, $\bar k =\max_{(x,\lambda,A)\in S}|k(x,\lambda,A)|$ and $\bar k' =\max_{(x,\lambda,A)\in S}|\tfrac{\partial k(x,\lambda,A)}{\partial\lambda}|$ exist.
	Since $(x,\lambda, V_{\Omega}(x))\in S$ for every $(x,\lambda)\in \partial\Omega\times (-\mu,\mu)$, it follows that
	$
	|J_{p-1}^{\partial\Omega}\psi_\Omega(x,\lambda)|\le \bar k
	$
	and
	$|\frac{\partial J_{p-1}^{\partial\Omega}\psi_\Omega(x,\lambda)}{\partial\lambda}|\le \bar k'$
	for every $(x,\lambda)\in \partial\Omega\times (-\mu,\mu)$.
\end{proof}

\subsection{Other Lemmas}

		\begin{lemma}\label{lemma:mean-plim}
			Fix any $a\in\{2,...,m\}$.
			Let $\{V_i\}_{i=1}^\infty$ be i.i.d. random variables such that $E[V_i^2]<\infty$.
			If Assumption \ref{consassump_complete} \ref{assumption:ML-continuity} -- \ref{finite-moments} hold, then for $l\ge 0$ and $m=0,1$,
			\begin{align*}
			E[V_i q_\delta^{ML}(a|X_i)^l 1\{q_\delta^{ML}(a|X_i)\in (0,1)\}^m1\{A_i\in\{1,a\}\}]\\
			\rightarrow E[V_i \widetilde{ML}(a|X_i)^l1\{\widetilde{ML}(a|X_i)\in (0,1)\} 1\{A_i\in\{1,a\}\}]
			\end{align*}
			as $\delta\rightarrow 0$.
			Moreover, if, in addition, $\delta_n\rightarrow 0$ as $n\rightarrow\infty$, then for $l\ge 0$,
			\begin{align*}
			\frac{1}{n}\sum_{i=1}^n V_i q_{\delta_n}^{ML}(a|X_i)^l 1\{q_{\delta_n}^{ML}(a|X_i)\in (0,1)\}1\{A_i\in\{1,a\}\}\\
			\stackrel{p}{\longrightarrow} E[V_i \widetilde{ML}(a|X_i)^l 1\{\widetilde{ML}(a|X_i)\in (0,1)\}1\{A_i\in\{1,a\}\}]
			\end{align*}
			as $n\rightarrow \infty$.
		\end{lemma}
	\begin{proof}
		Note that
		\begin{align*}
		E[\frac{1}{n}\sum_{i=1}^n V_i q_{\delta_n}^{ML}(a|X_i)^l 1\{q_{\delta_n}^{ML}(a|X_i)\in (0,1)\}1\{A_i\in\{1,a\}\}]\\
		=E[V_i q_{\delta_n}^{ML}(a|X_i)^l 1\{q_{\delta_n}^{ML}(a|X_i)\in (0,1)\}1\{A_i\in\{1,a\}\}].
		\end{align*}
		We show that
		\begin{align*}
		E[V_i q_\delta^{ML}(a|X_i)^l 1\{q_\delta^{ML}(a|X_i)\in (0,1)\}^m1\{A_i\in\{1,a\}\}]\\
		\rightarrow E[V_i \widetilde{ML}(a|X_i)^l 1\{\widetilde{ML}(a|X_i)\in (0,1)\}^m1\{A_i\in\{1,a\}\}]
		\end{align*}
		for $l\ge 0$ and $m=0,1$ as $\delta\rightarrow 0$, and that
		\begin{align*}
		&\Var(\frac{1}{n}\sum_{i=1}^n V_i q_{\delta_n}^{ML}(a|X_i)^l 1\{q_{\delta_n}^{ML}(a|X_i)\in (0,1)\}1\{A_i\in\{1,a\}\})\rightarrow 0
		\end{align*}
		for $l\ge 0$ as $n\rightarrow \infty$.
		For the first part, we have
		\begin{align*}
			&~E[V_i q_\delta^{ML}(a|X_i)^l 1\{q_\delta^{ML}(a|X_i)\in (0,1)\}^m1\{A_i\in\{1,a\}\}]\\
			=&~E[E[V_i|X_i,A_i] q_\delta^{ML}(a|X_i)^l 1\{q_\delta^{ML}(a|X_i)\in (0,1)\}^m1\{A_i\in\{1,a\}\}]\\
			=&~E[\sum_{a'\in\{1,a\}}E[V_i|X_i,A_i=a'] q_\delta^{ML}(a|X_i)^l 1\{q_\delta^{ML}(a|X_i)\in (0,1)\}^m ML(a'|X_i)]\\
			=&~ \int_{{\cal X}_{a,1}} g(x) q_\delta^{ML}(a|x)^l 1\{q_\delta^{ML}(a|x)\in (0,1)\}^mf_X(x)dx,
		\end{align*}
		where $g(x)=\sum_{a'\in\{1,a\}}E[V_i|X_i=x,A_i=a']ML(a'|x)$.
		
		Suppose $ML(a|\cdot)$ and $ML(1|\cdot)$ are continuous at $x$ and $\widetilde{ML}(a|x)\in (0,1)$. Then, with change of variables $u=\frac{x^*-x}{\delta}$, for $a'\in\{1,a\}$,
		\begin{align*}
		    p_\delta^{ML}(a'|x)&=\frac{\int_{B(x,\delta)}ML(a'|x^*)dx^*}{\int_{B(x,\delta)}dx^*}\\
		    &=\frac{\delta^p\int_{B(\bm{0},1)}ML(a'|x+\delta u)du}{\delta^p\int_{B(\bm{0},1)}du}\\
		    &\rightarrow \frac{\int_{B(\bm{0},1)}ML(a'|x)du}{\delta^p\int_{B(\bm{0},1)}du}=ML(a'|x)
		\end{align*}
		as $\delta\rightarrow 0$, where the convergence follows from the Dominated Convergence Theorem.
		It follows that
		$\lim_{\delta\rightarrow 0}q_\delta^{ML}(a|x)=\frac{ML(a|x)}{ML(a|x)+ML(1|x)}=\widetilde{ML}(a|x)\in (0,1)$, and hence $q_\delta^{ML}(a|x)\in (0,1)$ for sufficiently small $\delta>0$.
		Therefore, $1\{q_\delta^{ML}(a|x)\in (0,1)\}\rightarrow 1 = 1\{\widetilde{ML}(a|x)\in (0,1)\}$ as $\delta\rightarrow 0$.
		
		Suppose $x\in {\rm int}({\cal X}_{a,1}^a)\cup {\rm int}({\cal X}_{a,1}^1)$. Then $B(x,\delta)\subset {\cal X}_{a,1}^a$ or $B(x,\delta)\subset {\cal X}_{a,1}^1$ for sufficiently small $\delta>0$ by the fact that ${\rm int}({\cal X}_{a,1}^a)$ and ${\rm int}({\cal X}_{a,1}^1)$ are open.
		Note that if $\widetilde{ML}(a|x')=1$, then $\widetilde{ML}(1|x')=0$. Hence if $B(x,\delta)\subset {\cal X}_{a,1}^a$, $p_\delta^{ML}(1|x)=0$ so $q_\delta^{ML}(a|x)=1$.
		Likewise, if $B(x,\delta)\subset {\cal X}_{a,1}^1$, $p_\delta^{ML}(a|x)=0$ so $q_\delta^{ML}(a|x)=0$.
		It follows that $1\{q_\delta^{ML}(a|x)\in (0,1)\}\rightarrow 0 = 1\{\widetilde{ML}(a|x)\in (0,1)\}$ as $\delta\rightarrow 0$.
		
		Since $ML(a|\cdot)$ and $ML(1|\cdot)$ are continuous at $x$ for almost every $x\in{\cal X}_{a,1}$ by Assumption \ref{consassump_complete} \ref{assumption:ML-continuity}, and either $\widetilde{ML}(a|x)\in (0,1)$ or $x\in {\rm int}({\cal X}_{a,1}^a)\cup {\rm int}({\cal X}_{a,1}^1)$ for almost every $x\in {\cal X}_{a,1}$ by Assumption \ref{consassump_complete} \ref{assumption:ML-measure-zero-boundary}, the above results imply that $\lim_{\delta\rightarrow 0}q_\delta^{ML}(a|x)=\widetilde{ML}(a|x)$ and $\lim_{\delta\rightarrow 0}1\{q_\delta^{ML}(a|x)\in (0,1)\}= 1\{\widetilde{ML}(a|x)\in (0,1)\}$ for almost every $x\in{\cal X}_{a,1}$.
		By the Dominated Convergence Theorem,
		\begin{align*}
			&~E[V_i q_\delta^{ML}(a|X_i)^l 1\{q_\delta^{ML}(a|X_i)\in (0,1)\}^m1\{A_i\in\{1,a\}\}]\\
			\rightarrow &~\int_{{\cal X}_{a,1}} g(x) \widetilde{ML}(a|x)^l 1\{\widetilde{ML}(a|x)\in (0,1)\}^mf_X(x)dx\\
			=&~E[V_i\widetilde{ML}(a|X_i)^l 1\{\widetilde{ML}(a|X_i)\in (0,1)\}^m1\{A_i\in\{1,a\}\}]
		\end{align*}
		as $\delta\rightarrow 0$.
		As for variance,
		\begin{align*}
			&~\Var(\frac{1}{n}\sum_{i=1}^n V_i q_{\delta_n}^{ML}(a|X_i)^l 1\{q_{\delta_n}^{ML}(a|X_i)\in (0,1)\}1\{A_i\in\{1,a\}\})\\
			\le&~ \frac{1}{n}E[V_i^2 q_{\delta_n}^{ML}(a|X_i)^{2l} (1\{q_{\delta_n}^{ML}(a|X_i)\in (0,1)\}1\{A_i\in\{1,a\}\})^2]\\
			\le&~ \frac{1}{n}E[V_i^2]\\
			\rightarrow&~ 0
		\end{align*}
		as $n\rightarrow \infty$.
	\end{proof}

\section{Proofs}\label{appendix:proof}
\subsection{Derivation of Equation (\ref{eq:value-expression})}\label{section:derivation}
\begin{align*}
    V(\pi)&=V(ML)+E[\sum_{a\in{\cal A}}E[Y(a)|X](\pi(a|X)-ML(a|X))]\\
    &=V(ML)+E[\sum_{a\in{\cal A}}(E[Y(a)|X]-E[Y(1)|X])(\pi(a|X)-ML(a|X))]\\
    &~~~~+E[E[Y(1)|X]\sum_{a\in{\cal A}}(\pi(a|X)-ML(a|X))]\\
    &=V(ML)+E[\sum_{a\in{\cal A}}\beta(a,1)(\pi(a|X)-ML(a|X))]\\
    &=V(ML)+E[\sum_{a=2}^m\beta(a,1)(\pi(a|X)-ML(a|X))],
\end{align*}
where we use Assumption \ref{constant} and the fact that 
\[
\sum_{a\in{\cal A}}(\pi(a|X)-ML(a|X))
=\sum_a \pi(a|X) - \sum_a ML(a|X)
= 1 - 1
=0
\]
for the third equality.
\qed

\subsection{Proof of Lemma \ref{proposition:id-means}}
Suppose that Assumption \ref{continuity} holds.
		Pick $a\in{\cal A}$ and $x\in{\rm int}({\cal X})$ such that $p^{ML}(a|x)>0$.
		If $ML(a|x)>0$, $E[Y|X=x,A=a]=E[Y(a)|X=x]$, since $A$ is independent of $Y(a)$ conditional on $X$.
		$E[Y(a)|X=x]$ is thus identified.
        Suppose $ML(a|x)=0$.
		Since $x\in {\rm int}({\cal X})$, $B(x,\delta)\subset {\cal X}$ for any sufficiently small $\delta>0$.
		Moreover, since $p^{ML}(a|x)=\lim_{\delta\rightarrow 0} p_\delta^{ML}(a|x) >0$, $p_\delta^{ML}(a|x)>0$ for any sufficiently small $\delta>0$.
		This implies that we can find a point $x_{\delta}\in B(x,\delta)(\subset{\cal X})$ such that $ML(a|x_{\delta})>0$ for any sufficiently small $\delta>0$, for otherwise $p_\delta^{ML}(a|x)=0$.
		Noting that $x_{\delta}\rightarrow x$ as $\delta\rightarrow 0$,
		\begin{align*}
		\lim_{\delta\rightarrow 0} E[Y|X=x_{\delta},A=a]&=\lim_{\delta\rightarrow 0} E[Y(a)|X=x_{\delta}]\\
		&=E[Y(a)|X=x],
		\end{align*}
		where the first equality follows from conditional independence and the second from Assumption \ref{continuity}.
		\qed

\begin{comment}
\subsection{Proof of Corollary \ref{proposition:id-value}}
Pick $\pi\in\Pi^{ML}$.
Observe that
\begin{align*}
V(\pi)&=E[\sum_{a\in{\cal A}}E[Y(a)|X]\pi(a|X)]\\
&=E[\sum_{a\in{\cal A}}E[Y(a)|X]\pi(a|X)1\{X\in{\rm int}({\cal X})\}]~~\text{(by Assumption \ref{assumption:boundary-support})}\\
&=E[\sum_{a\in{\cal A}}E[Y(a)|X]\pi(a|X)1\{\pi(a|X)>0,X\in{\rm int}({\cal X})\}].
\end{align*}
Since $p^{ML}(a|x)>0$ if $\pi(a|x)>0$, Lemma \ref{proposition:id-means} implies that $E[Y(a)|X=x]$ is identified for every $(a,x)$ such that $\pi(a|x)>0$ and $x\in{\rm int}({\cal X})$.
$V(\pi)$ is thus identified.
\qed
\end{comment}

\subsection{Proof of Proposition \ref{proposition:id-value-every-policy}}
We show that $E[Y(a)|X=x]$ is identified for every $(a,x)$ pair.
Since $E[Y(a)|X=x]$ is identified for at least one $a\in{\cal A}$ for every $x\in{\cal X}$, and $E[Y(a')|X=x]=E[Y(a)|X=x]+\beta(a',a)$ by Assumption \ref{constant}, it suffices to show that $\beta(a',a)$ is identified for every $(a',a)$ pair.
This is equivalent to proving that $\beta(a,1)$ is identified for every $a\in\{2,...,m\}$, since $\beta(a',a)=\beta(a',1)-\beta(a,1)$.

Take any $a\in\{2,...,m\}$ and let $\{a_1,...,a_L\}$ be the sequence that satisfies the condition in Assumption \ref{path-existence}.
Under Assumption \ref{continuity}, Lemma \ref{proposition:id-means} implies that for every $l\in\{1,...,L-1\}$, $E[Y(a_{l+1})|X=x]$ and $E[Y(a_l)|X=x]$ are identified for some $x\in{\rm int}({\cal X})$.
This implies that $\beta(a_{l+1},a_l)$ is identified for every $l\in\{1,...,L-1\}$ under Assumption \ref{constant}.
Since
\begin{align*}
    \beta(a,1)&=\beta(a_L,a_{L-1})+\beta(a_{L-1},a_{L-2})+\cdots+\beta(a_{2},a_1),
\end{align*}
$\beta(a,1)$ is also identified.
\qed

\subsection{Proof of Theorem \ref{cor:opeconsistency}}\label{proof:opeconsistency}
Fix any $a\in\{2,...,m\}$ and consider the regression from the subsample assigned to either action $a$ or $1$ throughout the proof.
For notational simplicity, we omit the argument $a$ from $\widetilde{ML}(a|x)$ and $q_{\delta}^{ML}(a|x)$ and denote them by $\widetilde{ML}(x)$ and $q_{\delta}^{ML}(x)$.
Let $\mathbf{Z}_i=(1,1\{A_i=a\},q_{\delta_n}^{ML}(X_i))'$, and $I_i=1\{q_{\delta_n}^{ML}(X_i)\in (0,1)\}$.
Let
$$
\hat\beta=\begin{pmatrix}\hat\alpha_a\\\hat\beta_a\\\hat\gamma_a\end{pmatrix}=(\sum_{i=1}^n\mathbf{Z}_i\mathbf{Z}_i'I_i1\{A_i\in \{1,a\}\})^{-1}\sum_{i=1}^n\mathbf{Z}_iY_iI_i1\{A_i\in \{1,a\}\}.
$$

Below, we prove that $\hat\beta_a$ converges in probability to $\beta(a,1)$.
The theorem then immediately follows.
Also, the proof of Step \ref{step:consistency} shows that if Assumption \ref{constant} does not hold for a deterministic logging policy, $\hat\beta_a$ converges in probability to
$$
\frac{\int_{\partial\Omega^*\cap{\cal X}_{a,1}}E[Y_i(a)-Y_i(1)|X_i=x]f_X(x)d{\cal H}^{p-1}(x)}{\int_{\partial\Omega^*\cap{\cal X}_{a,1}}f_X(x)d{\cal H}^{p-1}(x)},
$$
which is the mean reward difference for the subpopulation on the decision boundary between a and 1.

We provide proofs separately for the two cases, the case in which $\Pr(\widetilde{ML}(X_i)\in (0,1))>0$ and the case in which $\Pr(\widetilde{ML}(X_i)\in (0,1))=0$.

\subsubsection{Consistency of $\hat\beta_{a}$ When $\Pr(\widetilde{ML}(X_i)\in (0,1)|A_i\in\{1,a\})>0$}
\label{subsubsection:consistency}
Let $\tilde{\mathbf{Z}}_i=(1,1\{A_i=a\},\widetilde{ML}(X_i))'$ and $I_i^{ML}=1\{\widetilde{ML}(X_i)\in (0,1)\}$.
By Lemma \ref{lemma:mean-plim},
\begin{align*}
\hat\beta&=(\sum_{i=1}^n \mathbf{Z}_i\mathbf{Z}_i'I_{i}1\{A_i\in \{1,a\}\})^{-1}\sum_{i=1}^n \mathbf{Z}_iY_iI_{i}1\{A_i\in \{1,a\}\}\\
&\stackrel{p}{\longrightarrow}(E[\tilde{\mathbf{Z}}_i\tilde{\mathbf{Z}}_i'I_i^{ML}1\{A_i\in \{1,a\}\}])^{-1}E[\tilde{\mathbf{Z}}_iY_iI_i^{ML}1\{A_i\in \{1,a\}\}]\\
&=(E[\tilde{\mathbf{Z}}_i\tilde{\mathbf{Z}}_i'I_i^{ML}|A_i\in \{1,a\}])^{-1}E[\tilde{\mathbf{Z}}_iY_iI_i^{ML}|A_i\in \{1,a\}]
\end{align*}
provided that $E[\tilde{\mathbf{Z}}_i\tilde{\mathbf{Z}}_i'I_i^{ML}|A_i\in \{1,a\}]$ is invertible.
After a few lines of algebra, we have
\begin{align*}
	&{\rm det}(E[\tilde{\mathbf{Z}}_i\tilde{\mathbf{Z}}_i'I_i^{ML}|A_i\in \{1,a\}])\\
	=&\Pr(I_i^{ML}=1|A_i\in \{1,a\})^2\Var(\widetilde{ML}(X_i)|I_i^{ML}=1,A_i\in \{1,a\})\\
	&\times E[\widetilde{ML}(X_i)(1-\widetilde{ML}(X_i))I_i^{ML}|A_i\in \{1,a\}]\\
	=&\Pr(I_i^{ML}=1|A_i\in \{1,a\})^2\Var(\widetilde{ML}(X_i)|I_i^{ML}=1,A_i\in \{1,a\})\\
	&\times E[\widetilde{ML}(X_i)(1-\widetilde{ML}(X_i))|A_i\in \{1,a\}].
\end{align*}
Therefore, $E[\tilde{\mathbf{Z}}_i\tilde{\mathbf{Z}}_i'I_i^{ML}|A_i\in \{1,a\}]$ is invertible, since $\Pr(I_i^{ML}=1|A_i\in \{1,a\})>0$, and $\Var(\widetilde{ML}(X_i)|I_i^{ML}=1,A_i\in \{1,a\})>0$ under Assumption \ref{consassump_complete} \ref{nonzero-cond-var}.

Another few lines of algebra gives
\begin{align*}
	(E[\tilde{\mathbf{Z}}_i\tilde{\mathbf{Z}}_i'I_i^{ML}|A_i\in \{1,a\}])^{-1}=\frac{1}{E[\widetilde{ML}(X_i)(1-\widetilde{ML}(X_i))|A_i\in \{1,a\}]}\begin{bmatrix} *~ & * & *\\
		0~ & 1 & -1\\
		*~ & * & *
	\end{bmatrix}.
\end{align*}
Observe that
\begin{align*}
    E[1\{A_i=a\}Y_{i}(a)|X_i,A_i\in\{1,a\}]&=E[1\{A_i=a\}|X_i,A_i\in\{1,a\}]E[Y_{i}(a)|X_i]\\
    &=\frac{\Pr(A_i=a|X_i)}{\Pr(A_i\in\{1,a\}|X_i)}E[Y_{i}(a)|X_i]\\
    &=\frac{ML(a|X_i)}{ML(a|X_i)+ML(1|X_i)}E[Y_{i}(a)|X_i]\\
    &=\widetilde{ML}(X_i)E[Y_{i}(a)|X_i],
\end{align*}
where the first equality follows from the assumption that $A_i$ is independent of $Y_i(\cdot)$ conditional on $X_i$.
Likewise,
$$
E[1\{A_i=1\}Y_{i}(1)|X_i,A_i\in\{1,a\}]=(1-\widetilde{ML}(X_i))E[Y_{i}(1)|X_i].
$$
Therefore, 
\begin{align*}
	\hat\beta_a\stackrel{p}{\longrightarrow}&~\frac{E[1\{A_i=a\}Y_iI_i^{ML}|A_i\in \{1,a\}]-E[\widetilde{ML}(X_i)Y_iI_i^{ML}|A_i\in \{1,a\}]}{E[\widetilde{ML}(X_i)(1-\widetilde{ML}(X_i))|A_i\in \{1,a\}]}\\
	=&~\frac{E[1\{A_i=a\}Y_{i}(a)I_i^{ML}-\widetilde{ML}(X_i)(1\{A_i=a\}Y_{i}(a)+1\{A_i=1\}Y_{i}(1))I_i^{ML}|A_i\in \{1,a\}]}{E[\widetilde{ML}(X_i)(1-\widetilde{ML}(X_i))|A_i\in \{1,a\}]}\\
	=&~\frac{E[\widetilde{ML}(X_i)E[Y_{i}(a)|X_i]I_i^{ML}-\widetilde{ML}(X_i)(\widetilde{ML}(X_i)E[Y_{i}(a)|X_i]|A_i\in \{1,a\}]}{E[\widetilde{ML}(X_i)(1-\widetilde{ML}(X_i)|A_i\in \{1,a\}]}\\
	&~+\frac{E[(1-\widetilde{ML}(X_i))E[Y_{i}(1)|X_i])I_i^{ML}|A_i\in \{1,a\}]}{E[\widetilde{ML}(X_i)(1-\widetilde{ML}(X_i)|A_i\in \{1,a\}]}\\
	=&~\frac{E[\widetilde{ML}(X_i)(1-\widetilde{ML}(X_i))E[Y_{i}(a)-Y_{i}(1)|X_i]I_i^{ML}|A_i\in \{1,a\}]}{E[\widetilde{ML}(X_i)(1-\widetilde{ML}(X_i)|A_i\in \{1,a\}]}\\
	=&~\beta(a,1).
\end{align*}
\qed

\subsubsection{Consistency of $\hat\beta_{a}$ When $\Pr(\widetilde{ML}(X_i)\in (0,1)|A_i\in\{1,a\})=0$}
\label{subsubsection:consistency-deterministic}

For notational simplicity, we omit subscript $a$ from $\Omega_a^*$ and denote it by $\Omega^*$.
We use the notation and results provided in Appendix \ref{appendix:notations-lemmas}.
By Lemma \ref{lemma:neighborhood-integral}, under Assumption \ref{consassump_complete} \ref{assumption:boundary-partition},
there exists $\mu>0$ such that $d_{\Omega^*}^s$ is twice continuously differentiable on $N(\partial\Omega^*,\mu)$ and that
$$
\int_{N(\partial \Omega^*,\delta)}g(x)dx = \int_{-\delta}^{\delta}\int_{\partial \Omega^*}g(u+\lambda \nu_{\Omega^*}(u))J_{p-1}^{\partial\Omega^*}\psi_{\Omega^*}(u,\lambda)d{\cal H}^{p-1}(u)d\lambda
$$
for every $\delta\in (0,\mu)$ and every function $g:\mathbb{R}^p\rightarrow \mathbb{R}$ that is integrable on $N(\partial \Omega^*,\delta)$.

Our proof proceeds in five steps.
\begin{step}\label{step:lim_qps}
	For every $(u,v)\in \partial\Omega^*\cap N({\cal X}_{a,1},\bar\delta)\times(-1,1)$, $\lim_{\delta\rightarrow 0}q_\delta^{ML}(u+\delta v \nu_{\Omega^*}(u))=k(v)
	$,
	where
	\begin{align*}
		k(v) =
		\begin{cases}
			1-\frac{1}{2}I_{(1-v^2)}(\frac{p+1}{2},\frac{1}{2}) & \ \ \ \text{for $v\in [0,1)$}\\
			\frac{1}{2}I_{(1-v^2)}(\frac{p+1}{2},\frac{1}{2}) & \ \ \ \text{for $v\in (-1,0)$}.
		\end{cases}
	\end{align*}
	Here $I_x(\alpha,\beta)$ is the regularized incomplete beta function (the cumulative distribution function of the beta distribution with shape parameters $\alpha$ and $\beta$).
\end{step}

\begin{proof}
	By Assumption \ref{consassump_complete} \ref{assumption:boundary-measure} \ref{assumption:zero-set}, there exists $\bar\delta \in (0,\frac{\mu}{2})$ such that $ML(a|x)=1$ or $ML(1|x)=1$ for almost every $x\in N({\cal X}_{a,1},3\bar\delta)\cap N(\partial\Omega^*,3\bar\delta)$.
	It follows that for $(u,v,\delta)\in \partial\Omega^*\cap N({\cal X}_{a,1},\bar\delta)\times (-1,1)\times(0,\bar\delta)$, $p_\delta^{ML}(a|u+\delta v \nu_{\Omega^*}(u))+p_\delta^{ML}(1|u+\delta v \nu_{\Omega^*}(u))=1$ so that $q_\delta^{ML}(u+\delta v \nu_{\Omega^*}(u))=p_\delta^{ML}(a|u+\delta v \nu_{\Omega^*}(u))$.
	
	For $(u,v,\delta)\in \partial\Omega^*\cap N({\cal X}_{a,1},\bar\delta)\times (-1,1)\times(0,\bar\delta)$,
	\begin{align*}
		&~p_\delta^{ML}(a|u+\delta v \nu_{\Omega^*}(u))\\
		=&~\frac{\int_{B(\bm{0},1)}ML(a|u+\delta v \nu_{\Omega^*}(u)+\delta w)dw}{\int_{B(\bm{0},1)}dw}\\
		=&~\frac{\int_{B(\bm{0},1)}1\{u+\delta v \nu_{\Omega^*}(u)+\delta w\in \Omega^*\}dw}{{\rm Vol}_p}\\
		=&~\frac{\int_{B(\bm{0},1)}1\{d_{\Omega^*}^s(u+\delta (v \nu_{\Omega^*}(u)+w))\ge 0)\}dw}{{\rm Vol}_p}\\
		=&~\frac{\int_{B(\bm{0},1)}1\{d_{\Omega^*}^s(u)+\nabla d_{\Omega^*}^s(y_d(u,\delta,v,w))'\delta(v \nu_{\Omega^*}(u)+w)\ge 0\}dw}{{\rm Vol}_p}\\
		=&~\frac{\int_{B(\bm{0},1)}1\{ \nu_{\Omega^*}(y_d(u,\delta,v,w))\cdot\delta(v \nu_{\Omega^*}(u)+w)\ge 0\}dw}{{\rm Vol}_p}\\
		=&~\frac{\int_{B(\bm{0},1)}1\{ \nu_{\Omega^*}(y_d(u,\delta,v,w))\cdot(v \nu_{\Omega^*}(u)+w)\ge 0\}dw}{{\rm Vol}_p},
	\end{align*}
	where ${\rm Vol}_p$ denotes the volume of the $p$-dimensional unit ball, the fourth equality follows by the mean value theorem with $y_d(u,\delta,v,w)$ on the line segment connecting $u$ with $u+\delta (v \nu_{\Omega^*}(u)+w)$, and the second last follows since $d_{\Omega^*}^s(u)=0$ for $u\in\partial\Omega^*$ and $\nabla d_{\Omega^*}^s(x)=\nu_{\Omega^*}(x)$ for $x\in N(\partial\Omega^*,\mu)$.
	Since $\lim_{\delta\rightarrow 0}y_d(u,\delta,v,w)=u$ and $\nu_{\Omega^*}$ is continuous, 
	$$
	\lim_{\delta\rightarrow 0} \nu_{\Omega^*}(y_d(u,\delta,v,w))\cdot(v \nu_{\Omega^*}(u)+w)=\nu_{\Omega^*}(u)\cdot(v \nu_{\Omega^*}(u)+w)=v+\nu_{\Omega^*}(u)\cdot w.
	$$
	Therefore,
	$$
	\lim_{\delta\rightarrow 0}1\{ \nu_{\Omega^*}(y_d(u,\delta,v,w))\cdot(v \nu_{\Omega^*}(u)+w)\ge 0\}=\begin{cases}
	    1 & ~~\text{if} ~~v+\nu_{\Omega^*}(u)\cdot w>0,\\
	    0 & ~~\text{if} ~~v+\nu_{\Omega^*}(u)\cdot w<0.
	\end{cases}
	$$
	By the Dominated Convergence Theorem,
	\begin{align*}
		\lim_{\delta\rightarrow 0}p_\delta^{ML}(a|u+\delta v \nu_{\Omega^*}(u))&=\frac{\int_{B(\bm{0},1)}1\{v+\nu_{\Omega^*}(u)\cdot w>0\}dw}{{\rm Vol}_p}.
	\end{align*}
	Note that the set $\{w\in B(\bm{0},1):v+\nu(u)\cdot w>0\}$ is a region of the $p$-dimensional unit ball cut off by the plane $\{w\in \mathbb{R}^p:v+\nu(u)\cdot w=0\}$.
	The distance from the center of the unit ball to the plane is $|v|$.
	Using the formula for the volume of a hyperspherical cap (see e.g. \cite{Li2011spherical-cap}), we have
	\begin{align*}
		\int_{B(\bm{0},1)}1\{v+\nu(u)\cdot w> 0\}dw =
		\begin{cases}
			{\rm Vol}_p-\frac{1}{2}{\rm Vol}_pI_{(2(1-v)-(1-v)^2)}(\frac{p+1}{2},\frac{1}{2}) & \ \ \ \text{for $v\in [0,1)$}\\
			\frac{1}{2}{\rm Vol}_pI_{(2(1+v)-(1+v)^2)}(\frac{p+1}{2},\frac{1}{2})  & \ \ \ \text{for $v\in (-1,0)$}.
		\end{cases}
	\end{align*}
	Therefore, $\lim_{\delta\rightarrow 0}p_\delta^{ML}(a|u+\delta v \nu_{\Omega^*}(u))=k(v)$.
\end{proof}

\begin{step}\label{step:nondegenerateQPS}
	For every $(u,v,\delta)\in \partial\Omega^*\cap N({\cal X}_{a,1},\bar\delta)\times(-1,1)\times (0,\bar\delta)$, $q^{ML}_\delta(u+\delta v\nu_{\Omega^*}(u))\in (0,1)$.
\end{step}

\begin{proof}
	Fix $(u,v,\delta)\in \partial\Omega^*\cap N({\cal X}_{a,1},\bar\delta)\times(-1,1)\times (0
	,\bar\delta)$.
	As discussed in \ref{step:lim_qps}, $q_\delta^{ML}(u+\delta v\nu_{\Omega^*}(u))=p_\delta^{ML}(a|u+\delta v\nu_{\Omega^*}(u))$.
	Suppose $v=0$.
	By Step \ref{step:lim_qps}, $p^{ML}(a|u)=\lim_{\delta'\rightarrow 0}p_{\delta'}^{ML}(a|u)=k(0)=\frac{1}{2}$.
	This implies that there exists $\delta'\in (0,\delta)$ such that $p_{\delta'}^{ML}(a|u)\in (0,1)$.
	It then follows that $0<{\cal L}^p(B(u,\delta')\cap \Omega^*)\le {\cal L}^p(B(x,\delta)\cap \Omega^*)$ and that $0<{\cal L}^p(B(x,\delta')\setminus \Omega^*)\le {\cal L}^p(B(x,\delta)\setminus \Omega^*)$.
	Therefore, $p^{ML}_{\delta}(a|u)=\frac{{\cal L}^p(B(u,\delta)\cap\Omega^*)}{{\cal L}^p(B(u,\delta))}\in (0,1)$.
	
	Suppose $v\neq 0$ and let $\epsilon\in (0, \delta(1- |v|))$.
	Note that $B(u,\epsilon)\subset B(u+\delta v\nu_{\Omega^*}(u),\delta)$, since for any $x\in B(u,\epsilon)$, $\|u+\delta v\nu_{\Omega^*}(u)-x\|\le \|\delta v\nu_{\Omega^*}(u)\|+\|u-x\|\le \delta|v|+\epsilon<\delta$.
	Since $p^{ML}(a|u)=\frac{1}{2}$, there exists $\epsilon'\in (0,\epsilon)$ such that $p^{ML}_{\epsilon'}(a|u)\in (0,1)$.
	It then follows that $0<{\cal L}^p(B(u,\epsilon')\cap \Omega^*)\le {\cal L}^p(B(u,\epsilon)\cap \Omega^*)\le {\cal L}^p(B(u+\delta v\nu_{\Omega^*}(u),\delta)\cap \Omega^*)$ and that $0<{\cal L}^p(B(x,\epsilon')\setminus \Omega^*)\le {\cal L}^p(B(x,\epsilon)\setminus \Omega^*)\le {\cal L}^p(B(u+\delta v\nu_{\Omega^*}(u),\delta)\setminus \Omega^*)$.
	Therefore, $p^{ML}_\delta(a|u+\delta v\nu_{\Omega^*}(u))=\frac{{\cal L}^p(B(u+\delta v\nu_{\Omega^*}(u),\delta)\cap \Omega^*)}{{\cal L}^p(B(u+\delta v\nu_{\Omega^*}(u),\delta))}\in (0,1)$.
\end{proof}

\begin{step}\label{step:mean-lim-2}
		Let $g:\mathbb{R}^p\rightarrow \mathbb{R}$ be a function that is bounded on $N(\partial\Omega^*,\delta')\cap N({\cal X}_{a,1},\delta')$ for some $\delta'>0$.
		Then, for $l\ge 0$, there exist $\tilde\delta>0$ and constant $C>0$ such that
		\begin{align*}
			|\delta^{-1} E[q_\delta^{ML}(X_i)^lg(X_i)1\{q_\delta^{ML}(X_i)\in (0,1)\}1\{A_i\in\{1,a\}\}]| \le C
		\end{align*}
		for every $\delta\in (0,\tilde\delta)$.
		If $g$ is continuous on $N(\partial\Omega^*,\delta')\cap N({\cal X}_{a,1},\delta')$ for some $\delta' >0$, then
		\begin{align*}
			&\delta^{-1} E[q_\delta^{ML}(X_i)^lg(X_i)1\{q_\delta^{ML}(X_i)\in (0,1)\}1\{A_i\in\{1,a\}\}]\\ &=\int_{-1}^1k(v)^ldv\int_{\partial\Omega^*\cap {\cal X}_{a,1}}g(x)f_X(x)d{\cal H}^{p-1}(x)+o(1),\\
			&\delta^{-1} E[1\{A_i=a\}q_\delta^{ML}(X_i)^lg(X_i)1\{q_\delta^{ML}(X_i)\in (0,1)\}]\\ &=\int_0^1k(v)^ldv\int_{\partial\Omega^*\cap{\cal X}_{a,1}}g(x)f_X(x)d{\cal H}^{p-1}(x)+o(1)
		\end{align*}
		for $l\ge 0$.
\end{step}

\begin{proof}
	Let $\bar\delta$ be given in Step \ref{step:lim_qps}.
	Under Assumption \ref{consassump_complete} \ref{assumption:boundary-continuity}, there exists $\tilde\delta\in (0,\bar\delta)$ such that $f_X$ is bounded and continuous on $N(\partial\Omega^*,2\tilde\delta)\cap N({\cal X}_{a,1},2\tilde\delta)$.
	Let $\tilde\delta\in (0,\bar\delta)$ be such that both $g$ and $f_X$ are bounded on $N(\partial\Omega^*,2\tilde\delta)\cap N({\cal X}_{a,1},2\tilde\delta)$ and such that $ML(a|x)=1$ or $ML(1|x)=1$ for almost every $x\in N({\mathcal{X}_{a,1}}, \tilde \delta)\cap N(\partial\Omega^*,\tilde \delta)$.
	Such $\tilde \delta$ exists under Assumption \ref{consassump_complete} \ref{assumption:boundary-measure} \ref{assumption:zero-set} and \ref{assumption:boundary-continuity}.
	
	We first show that $q_\delta^{ML}(x)\in \{0,1\}$ for every $x\in {\cal X}_{a,1}\setminus N(\partial\Omega^*,\delta)$ for every $\delta\in (0,\tilde\delta)$.
	Pick $x\in {\cal X}_{a,1}\setminus N(\partial\Omega^*,\delta)$ and $\delta\in (0,\tilde\delta)$.
	Since $B(x,\delta)\cap \partial\Omega^*=\emptyset$,
	either $B(x,\delta)\subset {\rm int}(\Omega^*)$ or $B(x,\delta)\subset{\rm int}(\mathbb{R}^p\setminus \Omega^*)$.
	If $B(x,\delta)\subset {\rm int}(\Omega^*)$, $q_\delta^{ML}(x)=1$.
	If $B(x,\delta)\subset{\rm int}(\mathbb{R}^p\setminus \Omega^*)$, $q_\delta^{ML}(x)=0$, since $ML(a|x')=0$ for all $x'\in \mathbb{R}^p\setminus \Omega^*$ by Assumption \ref{consassump_complete} \ref{assumption:deterministic}.
	Therefore, $\{x\in{\cal X}_{a,1}:q_\delta^{ML}(x)\in (0,1)\}\subset N(\partial\Omega^*,\delta)$ for every $\delta\in (0,\tilde\delta)$.
	This implies that $ML(a|x')=1$ or $ML(1|x')=1$ for almost every $x'\in\{x\in{\cal X}_{a,1}:q_\delta^{ML}(x)\in (0,1)\}$, since $ML(a|x)=1$ or $ML(1|x)=1$ for almost every $x\in N({\mathcal{X}_{a,1}}, \tilde \delta)\cap N(\partial\Omega^*,\tilde \delta)$.
	
	Using this result and Lemma \ref{lemma:neighborhood-integral}, for $\delta\in (0,\tilde\delta)$,
	\begin{align*}
		&~\delta^{-1} E[q_\delta^{ML}(X_i)^lg(X_i)1\{q_\delta^{ML}(X_i)\in (0,1)\}1\{A_i\in\{1,a\}\}]\\
		=&~\delta^{-1} E[q_\delta^{ML}(X_i)^lg(X_i)1\{q_\delta^{ML}(X_i)\in (0,1)\}(ML(a|X_i)+ML(1|X_i))1\{X_i\in{\cal X}_{a,1}\}]\\
		=&~\delta^{-1} E[q_\delta^{ML}(X_i)^lg(X_i)1\{q_\delta^{ML}(X_i)\in (0,1)\}1\{X_i\in{\cal X}_{a,1}\}]\\
		=&~\delta^{-1}\int q_\delta^{ML}(x)^lg(x)1\{q_\delta^{ML}(x)\in (0,1)\}f_X(x)1\{x\in{\cal X}_{a,1}\}dx\\
		=&~\delta^{-1}\int_{N(\partial\Omega^*,\delta)} q_\delta^{ML}(x)^lg(x)1\{q_\delta^{ML}(x)\in (0,1)\}f_X(x)1\{x\in{\cal X}_{a,1}\}dx\\
		=&~\delta^{-1}\int_{-\delta}^{\delta}\int_{\partial\Omega^*}q_\delta^{ML}(u+\lambda\nu_{\Omega^*}(u))^lg(u+\lambda\nu_{\Omega^*}(u))1\{q_\delta^{ML}(u+\lambda\nu_{\Omega^*}(u))\in (0,1)\}\\
		&~~~~~~~~~~~~~~~\times f_X(u+\lambda\nu_{\Omega^*}(u))1\{u+\lambda\nu_{\Omega^*}(u)\in{\cal X}_{a,1}\}J_{p-1}^{\partial\Omega^*}\psi_{\Omega^*}(u,\lambda)d{\cal H}^{p-1}(u)d\lambda.
	\end{align*}
	With change of variables $v = \frac{\lambda}{\delta}$, we have
	\begin{align*}
		&\delta^{-1} E[q_\delta^{ML}(X_i)^lg(X_i)1\{q_\delta^{ML}(X_i)\in (0,1)\}1\{A_i\in\{1,a\}\}]\\
		=&\int_{-1}^{1}\int_{\partial\Omega^*} q_\delta^{ML}(u+\delta v \nu_{\Omega^*}(u))^l 1\{q_\delta^{ML}(u+\delta v \nu_{\Omega^*}(u))\in (0,1)\}1\{u+\delta v \nu_{\Omega^*}(u)\in{\cal X}_{a,1}\}\\
		&~~~~~~~~~~~~~~~\times g(u+\delta v \nu_{\Omega^*}(u))f_X(u+\delta v \nu_{\Omega^*}(u))J_{p-1}^{\partial\Omega^*}\psi_{\Omega^*}(u,\delta v)d{\cal H}^{p-1}(u) dv.
	\end{align*}
	For every $(u,v,\delta)\in \partial\Omega^*\setminus N({\cal X}_{a,1},\tilde\delta)\times(-1,1)\times (0
	,\tilde\delta)$, $u+\delta v \nu_{\Omega^*}(u)\notin {\cal X}_{a,1}$, so
	\begin{align*}
		&\delta^{-1} E[q_\delta^{ML}(X_i)^lg(X_i)1\{q_\delta^{ML}(X_i)\in (0,1)\}1\{A_i\in\{1,a\}\}]\\
		=&\int_{-1}^{1}\int_{\partial\Omega^*\cap N({\cal X}_{a,1},\tilde\delta)} q_\delta^{ML}(u+\delta v \nu_{\Omega^*}(u))^l 1\{q_\delta^{ML}(u+\delta v \nu_{\Omega^*}(u))\in (0,1)\}\\
		&\times 1\{u+\delta v \nu_{\Omega^*}(u)\in{\cal X}_{a,1}\}g(u+\delta v \nu_{\Omega^*}(u))f_X(u+\delta v \nu_{\Omega^*}(u))J_{p-1}^{\partial\Omega^*}\psi_{\Omega^*}(u,\delta v)d{\cal H}^{p-1}(u) dv\\
		=&\int_{-1}^{1}\int_{\partial\Omega^*\cap N({\cal X}_{a,1},\tilde\delta)} q_\delta^{ML}(u+\delta v \nu_{\Omega^*}(u))^l1\{u+\delta v \nu_{\Omega^*}(u)\in{\cal X}_{a,1}\}\\ &~~~~~~~~~~~~~~~\times g(u+\delta v \nu_{\Omega^*}(u))f_X(u+\delta v \nu_{\Omega^*}(u))J_{p-1}^{\partial\Omega^*}\psi_{\Omega^*}(u,\delta v)d{\cal H}^{p-1}(u) dv\\
\end{align*}
	where the second equality follows from Step \ref{step:nondegenerateQPS}.
	By Lemma \ref{lemma:neighborhood-integral}, $J_{p-1}^{\partial\Omega^*}\psi_{\Omega^*}(\cdot,\cdot)$ is bounded on $\partial\Omega^*\times (-\tilde\delta,\tilde\delta)$.
	Since $g$ and $f_X$ are also bounded, for some constant $C>0$,
	\begin{align*}
		&|\delta^{-1} E[q_\delta^{ML}(X_i)^lg(X_i)1\{q_\delta^{ML}(X_i)\in (0,1)\}1\{A_i\in\{1,a\}\}]|\\
		&\le C \int_{-1}^{1}\int_{\partial\Omega^*\cap N({\cal X}_{a,1},\tilde\delta)} d{\cal H}^{p-1}(u) dv,
	\end{align*}
	which is finite by Assumption \ref{consassump_complete} \ref{assumption:boundary-measure} \ref{assumption:p-1dim}.
	
	Now suppose that $g$ and $f_X$ are continuous on $N(\partial\Omega^*,2\tilde \delta)\cap N({\cal X}_{a,1},2\tilde \delta)$.
	We can write
    \begin{align*}
		&\delta^{-1} E[q_\delta^{ML}(X_i)^lg(X_i)1\{q_\delta^{ML}(X_i)\in (0,1)\}1\{A_i\in\{1,a\}\}]\\
		=&\int_{-1}^{1}\int_{\partial\Omega^*\cap {\rm int}({\cal X}_{a,1})} q_\delta^{ML}(u+\delta v \nu_{\Omega^*}(u))^l1\{u+\delta v \nu_{\Omega^*}(u)\in{\cal X}_{a,1}\}\\ &~~~~~~~~~~~~~~~\times g(u+\delta v \nu_{\Omega^*}(u))f_X(u+\delta v \nu_{\Omega^*}(u))J_{p-1}^{\partial\Omega^*}\psi_{\Omega^*}(u,\delta v)d{\cal H}^{p-1}(u) dv\\
		&+\int_{-1}^{1}\int_{\partial\Omega^*\cap \partial{\cal X}_{a,1}} q_\delta^{ML}(u+\delta v \nu_{\Omega^*}(u))^l1\{u+\delta v \nu_{\Omega^*}(u)\in{\cal X}_{a,1}\}\\ &~~~~~~~~~~~~~~~\times g(u+\delta v \nu_{\Omega^*}(u))f_X(u+\delta v \nu_{\Omega^*}(u))J_{p-1}^{\partial\Omega^*}\psi_{\Omega^*}(u,\delta v)d{\cal H}^{p-1}(u) dv\\
		&+\int_{-1}^{1}\int_{\partial\Omega^*\cap (N({\cal X}_{a,1},\tilde\delta)\setminus{\rm cl}({\cal X}_{a,1}))} q_\delta^{ML}(u+\delta v \nu_{\Omega^*}(u))^l1\{u+\delta v \nu_{\Omega^*}(u)\in{\cal X}_{a,1}\}\\ &~~~~~~~~~~~~~~~\times g(u+\delta v \nu_{\Omega^*}(u))f_X(u+\delta v \nu_{\Omega^*}(u))J_{p-1}^{\partial\Omega^*}\psi_{\Omega^*}(u,\delta v)d{\cal H}^{p-1}(u) dv.
	\end{align*}
	The second term is zero by Assumption \ref{consassump_complete} \ref{assumption:boundary-measure} \ref{assumption:p-1dim}.
	Observe that $u+\delta v \nu_{\Omega^*}(u)\in{\cal X}_{a,1}$ for any sufficiently small $\delta>0$ if $u\in{\rm int}({\cal X}_{a,1})$ and that $u+\delta v \nu_{\Omega^*}(u)\notin{\cal X}_{a,1}$ for any sufficiently small $\delta>0$ if $u\notin{\rm cl}({\cal X}_{a,1})$.
	Therefore, by the Dominated Convergence Theorem,
	\begin{align*}
		&~\delta^{-1} E[q_\delta^{ML}(X_i)^lg(X_i)1\{q_\delta^{ML}(X_i)\in (0,1)\}1\{A_i\in\{1,a\}\}]\\ \rightarrow&\int_{-1}^1\int_{\partial\Omega^*\cap {\rm int}({\cal X}_{a,1})}k(v)^lg(u)f_X(u)J_{p-1}^{\partial\Omega^*}\psi_{\Omega^*}(u,0)d{\cal H}^{p-1}(u)dv\\
		=&\int_{-1}^1k(v)^ldv\int_{\partial\Omega^*\cap {\cal X}_{a,1}}g(u)f_X(u)d{\cal H}^{p-1}(u),
	\end{align*}
	where we use the fact from Lemma \ref{lemma:neighborhood-integral} that $J_{p-1}^{\partial\Omega^*}\psi_{\Omega^*}(u,\lambda)$ is continuous in $\lambda$ and $J_{p-1}^{\partial\Omega^*}\psi_{\Omega^*}(u,0)=1$.
	
	Now note that $ML(a|x)=1$ for every $x\in\Omega^*$ and $ML(a|x)=0$ for almost every $x\in N({\cal X}_{a,1},2\tilde\delta)\setminus\Omega^*$.
	Also, for every $(u,v,\delta)\in\partial\Omega^*\cap N({\cal X}_{a,1},\tilde\delta)\times(-1,1)\times (0
	,\tilde\delta)$, $u+\delta v \nu_{\Omega^*}(u)\in \Omega^*$ if $v\in (0,1)$ and $u+\delta v \nu_{\Omega^*}(u)\in N({\cal X}_{a,1},2\tilde\delta)\setminus\Omega^*$ if $v\in (-1,0]$.
	Therefore,
	\begin{align*}
		&~\delta^{-1} E[1\{A_i=a\}q_\delta^{ML}(X_i)^lg(X_i)1\{q_\delta^{ML}(X_i)\in (0,1)\}]\\
		=&~\delta^{-1} E[ML(a|X_i)q_\delta^{ML}(X_i)^lg(X_i)1\{q_\delta^{ML}(X_i)\in (0,1)\}1\{X_i\in{\cal X}_{a,1}\}]\\
		=&\int_{-1}^{1}\int_{\partial\Omega^*\cap N({\cal X}_{a,1},\tilde\delta)} ML(a|u+\delta v \nu_{\Omega^*}(u))q_\delta^{ML}(u+\delta v \nu_{\Omega^*}(u))^l1\{u+\delta v \nu_{\Omega^*}(u)\in{\cal X}_{a,1}\}\\ &~~~~~~~~~~~~~~~\times g(u+\delta v \nu_{\Omega^*}(u))f_X(u+\delta v \nu_{\Omega^*}(u))J_{p-1}^{\partial\Omega^*}\psi_{\Omega^*}(u,\delta v)d{\cal H}^{p-1}(u) dv\\
		=&\int_{0}^{1}\int_{\partial\Omega^*\cap N({\cal X}_{a,1},\tilde\delta)} q_\delta^{ML}(u+\delta v \nu_{\Omega^*}(u))^l1\{u+\delta v \nu_{\Omega^*}(u)\in{\cal X}_{a,1}\}\\ &~~~~~~~~~~~~~~~\times g(u+\delta v \nu_{\Omega^*}(u))f_X(u+\delta v \nu_{\Omega^*}(u))J_{p-1}^{\partial\Omega^*}\psi_{\Omega^*}(u,\delta v)d{\cal H}^{p-1}(u) dv\\
		\rightarrow&\int_{0}^1k(v)^ldv\int_{\partial\Omega^*\cap{\cal X}_{a,1}}g(u)f_X(u)d{\cal H}^{p-1}(u).
	\end{align*}
\end{proof}

\begin{step}\label{step:consistency}
	Let
	$$S_{\mathbf{Z}}=\lim_{\delta\rightarrow 0} \delta^{-1}E[\mathbf{Z}_i\mathbf{Z}_i' 1\{q_\delta^{ML}(X_i)\in (0,1)\}1\{A_i\in\{1,a\}\}]
	$$ and
	$$
	S_Y=\lim_{\delta\rightarrow 0} \delta^{-1}E[\mathbf{Z}_iY_i 1\{q_\delta^{ML}(X_i)\in (0,1)\}1\{A_i\in\{1,a\}\}].
	$$
	Then the second element of $S_{\mathbf{Z}}^{-1}S_Y$ is
	$$
	\frac{\int_{\partial\Omega^*\cap{\cal X}_{a,1}}E[Y_i(a)-Y_i(1)|X_i=x]f_X(x)d{\cal H}^{p-1}(x)}{\int_{\partial\Omega^*\cap{\cal X}_{a,1}}f_X(x)d{\cal H}^{p-1}(x)}.
	$$
	Under Assumption \ref{constant}, this is equal to $\beta(a,1)$.
\end{step}

\begin{proof}
	Note that
	\begin{align*}
	&~E[\mathbf{Z}_iY_i 1\{q_\delta^{ML}(X_i)\in (0,1)\}1\{A_i\in\{1,a\}\}]\\
	=&~E[\mathbf{Z}_i(1\{A_i=a\}Y_i(a)+1\{A_i=1\}Y_i(1))1\{q_\delta^{ML}(X_i)\in (0,1)\}]\\
	=&~E[\mathbf{Z}_i(E[1\{A_i=a\}|X_i]E[Y_i(a)|X_i]+E[1\{A_i=1\}|X_i]E[Y_i(1)|X_i])1\{q_\delta^{ML}(X_i)\in (0,1)\}]\\
	=&~E[\mathbf{Z}_i(1\{A_i=a\}E[Y_i(a)|X_i]+1\{A_i=1\}E[Y_i(1)|X_i])1\{q_\delta^{ML}(X_i)\in (0,1)\}],
	\end{align*}
	where the second equality holds since $A_i$ is independent of $Y_i(\cdot)$ conditional on $X_i$.
	By Step \ref{step:mean-lim-2},
	\begin{align*}
	S_{\mathbf{Z}}
	=&
	\bar f_X\begin{bmatrix}
	2& 1  & \int_{-1}^1 k(v)dv\\
	1 & 1 & \int_{0}^1 k(v)dv \\
	\int_{-1}^1 k(v)dv & \int_{0}^1 k(v)dv & \int_{-1}^1 k(v)^2dv
	\end{bmatrix},
	\end{align*}
	where $\bar f_X=\int_{\partial\Omega^*\cap{\cal X}_{a,1}}f_X(x)d{\cal H}^{p-1}(x)$, and
	\begin{align*}
	S_Y=
	\begin{bmatrix}
	\int_{\partial\Omega^*\cap{\cal X}_{a,1}}E[Y_{i}(a)+Y_{i}(1)|X_i=x]f_X(x)d{\cal H}^{p-1}(x) \\
	\int_{\partial\Omega^*\cap{\cal X}_{a,1}}E[Y_{i}(a)|X_i=x]f_X(x)d{\cal H}^{p-1}(x)  \\
	\int_{\partial\Omega^*\cap{\cal X}_{a,1}}(\int_{0}^1 k(v)dvE[Y_{i}(a)|X_i=x] + \int_{-1}^0 k(v)dvE[Y_{i}(1)|X_i=x])f_X(x)d{\cal H}^{p-1}(x)
	\end{bmatrix}.
	\end{align*}
	After a few lines of algebra, we have
	\begin{align*}
		{\rm det}(S_{\mathbf{Z}})
	=&\bar f_X^{-1}(\int_{-1}^0(k(v)-\int_{-1}^0k(s)ds)^2dv+\int_{0}^1(k(v)-\int_{0}^1k(s)ds)^2dv),
	\end{align*}
	which is nonzero under Assumption \ref{consassump_complete} \ref{assumption:boundary-measure} \ref{assumption:p-1dim}.
	After another few lines of algebra, we obtain that the second element of $S_{\mathbf{Z}}^{-1}S_Y$ is 
	$$
	\frac{\int_{\partial\Omega^*\cap{\cal X}_{a,1}}E[Y_i(a)-Y_i(1)|X_i=x]f_X(x)d{\cal H}^{p-1}(x)}{\int_{\partial\Omega^*\cap{\cal X}_{a,1}}f_X(x)d{\cal H}^{p-1}(x)}=\beta(a,1).
	$$
	Note that if Assumption \ref{constant} does not hold, the left-hand side still represents the mean reward difference for the subpopulation on the boundary $\partial\Omega^*\cap{\cal X}_{a,1}$.
\end{proof}

\begin{step}\label{step:consistency-2}
	If $n\delta_n\rightarrow \infty$ as $n\rightarrow \infty$, then $\hat\beta_a\stackrel{p}{\longrightarrow}\beta(a,1)$.
\end{step}

\begin{proof}
	It suffices to verify that the variance of each element of $\frac{1}{n\delta_n}\sum_{i=1}^n\mathbf{Z}_i\mathbf{Z}_i' I_{i}1\{A_i\in\{1,a\}\}$ and $\frac{1}{n\delta_n}\sum_{i=1}^n\mathbf{Z}_iY I_{i}1\{A_i\in\{1,a\}\}$ is $o(1)$.
	Here, we only verify that $\Var(\frac{1}{n\delta_n}\sum_{i=1}^nq_{\delta_n}^{ML}(X_i)Y_iI_{i}1\{A_i\in\{1,a\}\})=o(1)$.
	Note that
	\begin{align*}
	&E[Y_i^21\{A_i\in\{1,a\}\}|X_i]=E[1\{A_i=a\}Y_{i}(a)^2+1\{A_i=1\}Y_{i}(1)^2|X_i]\\
	&\le E[Y_{i}(a)^2+Y_{i}(1)^2|X_i]E[1\{A_i\in\{1,a\}\}|X_i].
	\end{align*}
	Under Assumption \ref{consassump_complete} \ref{assumption:boundary-continuity}, there exists $\delta'>0$ such that $E[Y_{i}(a)^2+Y_{i}(1)^2|X_i]$ is bounded on $N(\partial\Omega^*,\delta')$.
	We have 
	\begin{align*}
		~&\Var(\frac{1}{n\delta_n}\sum_{i=1}^nq_{\delta_n}^{ML}(X_i)Y_iI_{i}1\{A_i\in\{1,a\}\})\\
		\le~& \frac{1}{n\delta_n}\delta_n^{-1}E[q_{\delta_n}^{ML}(X_i)^2Y_i^2I_{i}1\{A_i\in\{1,a\}\}]\\
		=~&\frac{1}{n\delta_n}\delta_n^{-1}E[q_{\delta_n}^{ML}(X_i)^2E[Y_{i}(a)^2+Y_{i}(1)^2|X_i]I_{i}1\{A_i\in\{1,a\}\}]\\
		\le~& \frac{1}{n\delta_n}C
	\end{align*}
	for some $C>0$, where the last inequality follows from Step \ref{step:mean-lim-2}.
	The conclusion follows since $n\delta_n\rightarrow \infty$.
\end{proof}
%%%%%%%%%%%%%%%%%%%%%%%%%%%%%%%%%%%%%%%%%%%%%%%%%%%%%%%%%%%%%%%%%%%%%%%%%%%%%%%
%%%%%%%%%%%%%%%%%%%%%%%%%%%%%%%%%%%%%%%%%%%%%%%%%%%%%%%%%%%%%%%%%%%%%%%%%%%%%%%

\end{document}